\documentclass{article}

     \usepackage[final]{neurips_2022}

\usepackage[utf8]{inputenc} 
\usepackage[T1]{fontenc}    
\usepackage{hyperref}       
\usepackage{url}            
\usepackage{booktabs}       
\usepackage{amsfonts}       
\usepackage{nicefrac}       
\usepackage{microtype}      
\usepackage{xcolor}         

\usepackage[utf8]{inputenc} 
\usepackage[T1]{fontenc}    
\usepackage{hyperref}       
\usepackage{url}            
\usepackage{booktabs}       
\usepackage{amsfonts}       
\usepackage{nicefrac}       
\usepackage{microtype}      
\usepackage{xcolor}         

\usepackage{microtype}
\usepackage{graphicx}
\usepackage{booktabs} 

\usepackage{microtype}
\usepackage{graphicx}
\usepackage{booktabs} 

\usepackage{amssymb, amsmath,amsthm}
\usepackage{amsfonts}
\usepackage{algorithm}
\usepackage{algorithmic}
\usepackage{paralist}
\usepackage{multirow}
\usepackage{wrapfig}

\usepackage{url,enumerate}
\usepackage{color,xcolor}
\usepackage{makeidx}  
\usepackage{amsmath,amssymb}
\usepackage{mathtools}
\usepackage[small, compact]{titlesec}
\usepackage{xspace}
\usepackage{epstopdf}
\usepackage{cite}
\usepackage{mathrsfs}
\usepackage{enumerate}
\usepackage{color}
\usepackage{graphicx,epsfig}
\usepackage{amsmath,amssymb,xspace}
\usepackage{url}
\usepackage{hyperref}
\usepackage{bm}
\usepackage{bbm}
\usepackage{upgreek}
\usepackage{cleveref}
\usepackage{multirow}
\usepackage{ulem}
\usepackage{cancel}
\usepackage{subcaption}
\usepackage{dsfont}
\usepackage{hyperref}

\newtheorem{theorem}{Theorem}[section]
\newtheorem{corollary}{Corollary}[section]

\newcommand{\cmt}[1]{}
\newcommand{\eg}{{\textit{e.g.},}\xspace}
\newcommand{\ie}{{\textit{i.e.},}\xspace}
\newcommand{\etc}{{\textit{etc}.}\xspace}

\newcommand{\dmf}{DMFAL\xspace}
\newcommand{\ours}{BMFAL-BC\xspace}
\newcommand{\dmfal}{DMFAL\xspace}
\newcommand{\dmfalbc}{DMFAL-BC\xspace}

\renewcommand{\vec}{{\rm vec}}

\newcommand{\h}{{\bf h}}

\newcommand{\x}{{\bf x}}
\newcommand{\y}{{\bf y}}

\newcommand{\A}{{\bf A}}

\newcommand{\Dcal}{\mathcal{D}}
\newcommand{\Qcal}{{\mathcal{Q}}}

\newcommand{\I}{{\bf I}}

\newcommand{\Mcal}{{\mathcal{M}}}
\newcommand{\Ocal}{{\mathcal{O}}}

\newcommand{\N}{\mathcal{N}}  

\newcommand{\W}{{\bf W}}

\newcommand{\Xcal}{{\mathcal{X}}}

\newcommand{\Ycal}{{\mathcal{Y}}}

\newcommand{\bphi}{\boldsymbol{\phi}}

\newcommand{\btheta}{\boldsymbol{\theta}}

\newcommand{\bxi}{\boldsymbol{\xi}}
\newcommand{\bSigma}{\boldsymbol{\Sigma}}

\newcommand{\bmu}{\boldsymbol{\mu}}

\newcommand{\0}{{\bf 0}}

\newcommand{\ben}{\begin{enumerate}}
\newcommand{\een}{\end{enumerate}}

\newcommand{\argmax}{\operatornamewithlimits{argmax}}

\newcommand{\EE}{\mathbb{E}}
\newcommand{\bbI}{\mathbb{I}}
\newcommand{\bbH}{\mathbb{H}}

\title{Batch Multi-Fidelity Active Learning with Budget Constraints}

\author{%
	Shibo Li\thanks{Equal contribution \;\; Correspondence to: Jeff M. Phillips, Shandian Zhe.}$^*$, Jeff M. Phillips$^*$, Xin Yu, Robert M. Kirby, and Shandian Zhe\\
	School of Computing, University of Utah\\
	Salt Lake City, UT 84112\\
	\texttt{\{shibo, jeffp, xiny, kirby, zhe\}@cs.utah.edu}
}

\begin{document}

\maketitle


\begin{abstract}
Learning functions with high-dimensional outputs is critical in many applications, such as physical simulation and engineering design. However, collecting training examples for these applications is often costly, \eg by running numerical solvers. The recent work \citep{li2020deep} proposes the first multi-fidelity active learning approach for high-dimensional outputs, which can acquire examples at different fidelities to reduce the cost while improving the learning performance. However,  this method only queries at one pair of fidelity and input at a time, and hence has a risk to bring in strongly correlated examples to reduce the learning efficiency. In this paper, we propose Batch Multi-Fidelity Active Learning with Budget Constraints (\ours), which can promote the diversity of training examples to improve the benefit-cost ratio, while respecting a given budget constraint for batch queries. Hence, our method can be more practically useful. Specifically, we propose a novel batch acquisition function that measures the mutual information between a batch of multi-fidelity queries and the target function, so as to penalize highly correlated queries and encourages diversity. 
The optimization of the batch acquisition function is challenging in that it involves a combinatorial search over many fidelities while subject to the budget constraint. To address this challenge, we develop a weighted greedy algorithm that can sequentially identify each (fidelity, input) pair, while achieving a near $(1 - 1/e)$-approximation of the optimum. We show the advantage of our method in several computational physics and engineering applications. 
\end{abstract}
\section{Introduction}
Applications, such as in computational physics and engineering design, often demand we calculate a complex mapping from low-dimensional inputs to high-dimensional outputs, such as finding the optimal material layout (output) given the design parameters (input), and solving the solution field  on a mesh (output) given the PDE parameters (input). Computing these mappings is often very expensive, \eg iteratively running numerical solvers. Hence, learning a surrogate model to outright predict the mapping, which is much faster and cheaper, is of great practical interest and importance~\citep{kennedy2000predicting,conti2010bayesian}

However, collecting training examples for the surrogate model becomes another bottleneck, since each example still requires a costly computation.  To alleviate this issue, \citet{li2020deep} developed DMFAL, a first deep  multi-fidelity active learning algorithm, which can acquire  examples at different fidelities to reduce the cost of data collection. Low-fidelity examples are cheap to compute (\eg with coarse meshes) yet inaccurate; high-fidelity examples are accurate but much more expensive (\eg calculated with dense grids). See Fig. \ref{fig:motivation} for an illustration. DMFAL uses an optimization-based acquisition method to dynamically identify the input and fidelity at which to query new examples, so as to improve the learning performance,  lower the sample complexity, and reduce the cost.

Despite its effectiveness, DMFAL can only optimize and query at one pair of input and fidelity each time  and hence ignores the correlation between consecutive queries. As a result, it has a risk of bringing in strongly correlated examples, which can restrict the learning efficiency and lead to a suboptimal benefit-cost ratio. In addition, the sequential querying and training strategy is difficult to utilize parallel computing resources that are common nowadays (\eg multi-core CPUs/GPUs and computer clusters) to query concurrently and to further speed up.

In this paper, we propose \ours, a batch multi-fidelity active learning method with budget constraints. Our method can acquire a batch of multi-fidelity examples at a time to inhibit the example correlations, promote diversity so as to improve the learning efficiency and benefit-cost ratio. Our method can respect a given budget in issuing batch queries, hence are more widely applicable and practically useful. Specifically, we first propose a novel acquisition function, which measures the mutual information between a batch of multi-fidelity queries and the target function. The acquisition function not only can  penalize highly correlated queries to encourage diversity, but also can be efficiently computed by an Monte-Carlo approximation. However, optimizing the acquisition function is challenging because it incurs a combinatorial search over fidelities and meanwhile needs to obey the constraint. To address this challenge, we develop a weighted greedy algorithm. We sequentially find one pair of fidelity and input each step, by maximizing the increment of the mutual information weighted by the cost. In this way, we avoid enumerating the fidelity combinations and greatly improve the efficiency. We prove that our greedy algorithm nearly achieves a $(1-\frac{1}{e})$-approximation of the optimum, with a few minor caveats.

For evaluation, we examined \ours in five real-world applications, including three benchmark tasks in physical simulation (solving Poisson's, Heat and viscous Burger's equations),  a topology structure design  problem,  and  a computational fluid dynamics (CFD) task to predict the velocity field of boundary-driven flows. We compared with the budget-aware version of DMFAL,  single multi-fideity querying  with our acquisition function,  and several random querying strategies. Under the same budget constraint, our method consistently outperforms the competing methods throughout the learning process, often by a large margin.
\section{Background}
\subsection{Problem Setting}
Suppose we aim to learn a mapping $f: \Omega \subseteq \mathbb{R}^r \rightarrow \mathbb{R}^d$ where  $r$ is small but   $d$ is  large, \eg hundreds of thousands. To economically learn this mapping, we collect training examples at $M$ fidelities. Each fidelity $m$ corresponds to mapping $f_m: \Omega \rightarrow \mathbb{R}^{d_m}$. The target mapping is computed at the highest fidelity, \ie $f(\x) = f_M(\x)$. The other $f_m$ can be viewed as a (rough) approximation of $f$. Note that $d_m$ is unnecessarily the same as $d$ for $m<M$. For example, solving PDEs on a coarse mesh will give a lower-dimensional output (on the mesh points). However, we can interpolate it to the  $d$-dimensional space to match $f(\cdot)$ (this is standard in physical simulation~\citep{zienkiewicz1977finite}). Denote by $\lambda_m$ the cost of computing $f_m(\cdot)$ at fidelity $m$. We have $\lambda_1 \le \ldots \le  \lambda_M$. 

\subsection{Deep Multi-Fidelity Active Learning (DMFAL)}\label{sect:model}
To effectively estimate $f$ while reducing the cost, \citet{li2020deep} proposed DMFAL, a multi-fidelity deep active learning approach. Specifically,  a neural network (NN) is introduced for each fidelity $m$, where a low-dimensional hidden output $\h_m(\x)$ is first generated, and then projected to the high-dimensional observation space. Each NN is parameterized by $\left(\A_m, \W_m, \btheta_m\right)$, where $\A_m$ is the projection matrix, $\W_m$ is the weight matrix of the last layer, and $\btheta_m$ consists of the remaining NN parameters. The model is defined as follows,
\begin{align}
	\x_m = [\x; \h_{m-1}(\x)], \;\;\h_m(\x) &= \W_m\bphi_{\btheta_m}(\x_m), \;\;\y_m(\x) = \A_m \h_m(\x) + \bxi_m, \label{eq:model}
\end{align}
where $\x_m$ is the input to the NN at fidelity $m$, $\y_m(\x)$ is the observed $d_m$ dimensional output, $\bxi_m\sim \N(\cdot|\0, \tau_m\I)$ is a random noise, $\bphi_{\btheta_m}(\x_m)$ is the output of the second last layer and can be viewed as a nonlinear feature transformation of $\x_m$. Since $\x_m$ includes not only the original input $\x$, but also the hidden output from the previous fidelity, \ie $\h_{m-1}(\x)$, the model can propagate information throughout fidelities and capture the complex relationships (\eg nonlinear and nonstationary) between different fidelities. The whole model is visualized in Fig. \ref{fig:graphical} of Appendix. To estimate the posterior of the model, DMFAL uses a structural variational inference algorithm. A multi-variate Gaussian posterior is introduced for each weight matrix, $q(\W_m) = \N\big(\vec(\W_m)|\bmu_m, \bSigma_m\big)$. A variational evidence lower bound (ELBO) is maximized via stochastic optimization and the reparameterization trick~\citep{kingma2013auto}.
\begin{figure}[t]
	\centering
	\includegraphics[width=1.0\textwidth]{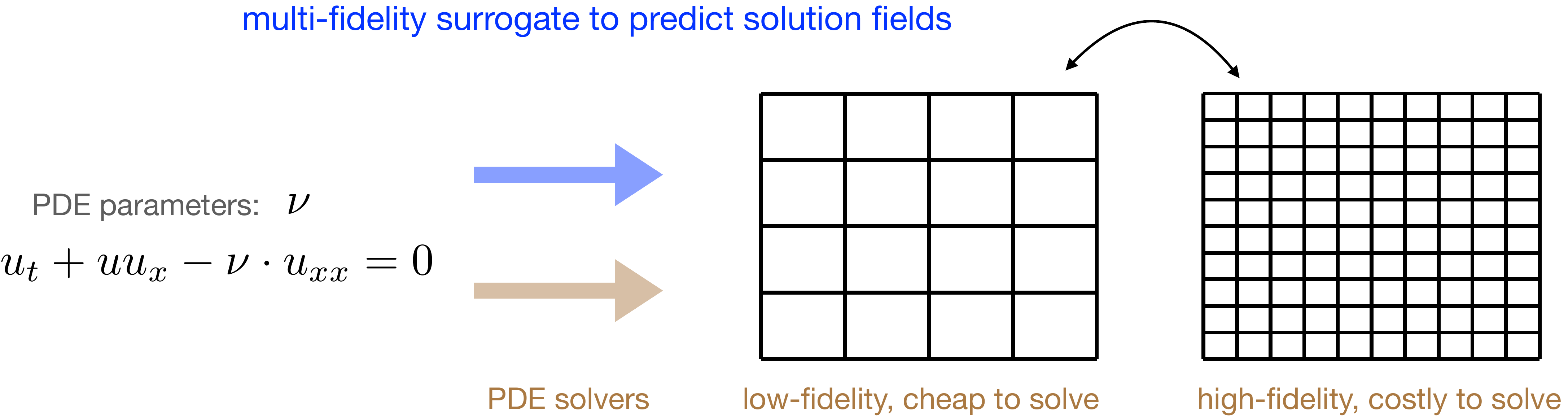}
	\caption{\small Illustration of the motivation and goal with physical simulation as an example. It is costly to solve every PDE from scratch. We aim to train a multi-fidelity surrogate model to directly predict high-fidelity solution fields given the PDE parameters. 
	To further reduce the cost of collecting the training data with numerical solvers, we seek to develop multi-fidelity active learning algorithms.} \label{fig:motivation}
\end{figure}

To conduct active learning,  DMFAL views the most valuable example at each fidelity $m$ as the one that can best help its prediction at the highest fidelity $M$ (\ie the target function). Accordingly, the acquisition function is defined as
\begin{align}
	&a(m,\x) = \frac{1}{\lambda_m} \bbI\big(\y_m(\x), \y_M(\x)|\Dcal\big) =\frac{1}{\lambda_m} \left(\bbH(\y_m |\Dcal) + \bbH(\y_M|\Dcal) - \bbH(\y_m, \y_M|\Dcal) \right), \label{eq:ac-single}
\end{align}
where $\bbI(\cdot, \cdot)$ is the mutual information, $\bbH(\cdot)$ is the entropy, and $\Dcal$ is the current training dataset.  The computation of the acquisition  function is quite challenging because $\y_m$ and $\y_M$ are both high dimensional. To address this issue, DMFAL takes advantage of the fact that each low-dimensional output $\h_m(\x)$ is a nonlinear function  of the random weight matrices $\{\W_1, \ldots, \W_m\}$.  Based on the variational posterior $\{q(\W_j)\}$,  DMFAL uses the multi-variate delta method~\citep{oehlert1992note,bickel2015mathematical} to estimate the mean and covariance of $\widehat{\h} = [\h_m; \h_M]$, with which  to construct a joint Gaussian posterior approximation for $\widehat{\h}$ by moment matching. Then, from the projection in \eqref{eq:model}, we can derive a Gaussian posterior for $[\y_m; \y_M]$. By further applying Weinstein-Aronszajn identify~\citep{kato2013perturbation}, we can compute the entropy terms in \eqref{eq:ac-single} conveniently and efficiently. 

Each time, DMFAL maximizes the acquisition function \eqref{eq:ac-single} to identify a pair of input and fidelity at which to query. The new example is added into $\Dcal$, and  the deep multi-fidelity model is retrained and updated. The process repeats until the maximum number of training examples have been queried or other stopping criteria are met.  

\section{Batch Multi-Fidelity Active Learning with Budget Constraints}
While effective,   DMFAL can only identify and query at one input and fidelity each time, hence ignoring the correlations between the successive queries. As a result,  highly correlated examples are easier to be acquired and incorporated into the training set.  Consider we have found $(\x^*, m^*)$ that maximizes the acquisition function \eqref{eq:ac-single}. It is often the case that a single example will not cause a significant change of the model posterior $\{q(\W_m)\}$ (especially when the current dataset $\Dcal$ is not very small).  When we optimize the acquisition function again, we are likely to obtain a new pair $(\hat{\x}^*, \hat{m}^*)$ that is close to $(\x^*, m^*)$ (\eg $\hat{m}^* = m^*$ and $\hat{\x}^*$ close to $\x^*$). Accordingly, the queried example will be highly correlated to the previous one.  The redundant information within these examples can restrict the learning efficiency, and demand for more queries (and cost) to achieve the satisfactory performance. 
Note that the similar issue have been raised in other single-fidelity, pool-based active learning works, \eg~\citep{geifman2017deep,kirsch2019batchbald}; see Sec \ref{sect:rel}. 

To overcome this problem, we propose \ours, a batch multi-fidelity active learning approach  to reduce the correlations and to promote the diversity of training examples, so that we can improve the learning performance, lower the sample complexity, and better save the cost. In addition, we take into account the budget constraint in querying the batch examples, which is common in practice (like cloud or high-performance computing)~\citep{mendes2020trimtuner}.  Under the given budget, we aim to find a batch of multi-fidelity queries that improves the benefit-cost ratio as much as possible.

\subsection{Batch Acquisition Function}

Specifically, we first consider an acquisition function that allows us to jointly optimize a set of inputs and fidelities. While it seems natural to consider how to extend \eqref{eq:ac-single} to the batch case, the acquisition function in \eqref{eq:ac-single} is about the mutual information between $\y_m(\x)$ and $\y_M(\x)$. That means, it only measures the utility of the query $(m, \x)$ in improving the estimate of the target function at $\x$ itself (\ie $\y_M(\x)$), rather than at any other input. 
To take into account the utility in improving the function estimation at all the inputs, we therefore propose a new single acquisition function, 
\begin{align}
	{a}_s(m, \x) = \EE_{p(\x')}\left[ \mathbb{I}\left(\y_m(\x), \y_M(\x')|\Dcal\right)\right] \label{eq:new-acqn}
\end{align}
where $\x'\in \Omega$ and $p(\x')$ is the distribution of the input in $\Omega$.  We can see that, by varying the input $\x'$ in the second argument of the mutual information, we are able to evaluate the utility of the query in improving the estimation of the entire body of the target function. Hence, it is more reasonable and comprehensive.  Now, consider querying a batch of examples under the budget $B$, we extend \eqref{eq:new-acqn} to
\begin{align}
	{a}_{\text{batch}}(\Mcal, \Xcal) ={\EE_{p(\x')}\left[\bbI\left(\{\y_{m_j}(\x_j) \}_{j=1}^n , \y_M(\x') |\Dcal\right)\right]}, \;\;\; \text{s.t.} \;\; {\sum_{j=1}^n \lambda_{m_j} \le B}, \label{eq:ac-batch}
\end{align}
where $\Mcal = \{m_1, \ldots, m_n\}$, $\Xcal = \{\x_1, \ldots, \x_n\}$, and $\Dcal$ is the current training dataset.  We can see that the more correlated $\{\y_{m_j}(\x_j)\}_{j=1}^n$, the smaller the mutual information, and hence the smaller the expectation in \eqref{eq:ac-batch}. Therefore, the batch acquisition function implicitly penalizes strongly correlated queries and favors diversity. 

The expected mutual information in \eqref{eq:ac-batch} is usually analytically intractable. However, we can efficiently compute it with an Monte-Carlo approximation. That is, we draw $A$ independent samples from the input distribution, $\x'_1, \ldots, \x'_A \sim p(\x')$, and compute 
\begin{align}
	\widehat{a}_{\text{batch}}(\Mcal, \Xcal) = \frac{1}{A} \sum_{l=1}^A \bbI\left(\{\y_{m_j}(\x_j) \}_{j=1}^n , \y_M(\x'_l) |\Dcal\right). \label{eq:ac-batch2} \;\; 
\end{align}
It is straightforward to extend the multi-variate delta method used in DMFAL to calculate the mutual information in \eqref{eq:ac-batch2}. We can then maximize \eqref{eq:ac-batch2} subject to the budget constraint, $\sum_{j=1}^B \lambda_{m_j} \le B$, to acquire more diverse and informative training examples. 
 In addition, when parallel computing resources (\eg multi-core CPUs/GPUs and computer clusters) are available, we can acquire these queries in parallel to further speed up the active learning.  

However, directly maximizing \eqref{eq:ac-batch2} will incur a combinatorial search over multiple fidelities in  $\Mcal$, and hence is very expensive. Note that the number of examples $n$ is not fixed, and can vary as long as the cost does not exceed the budget $B$, which makes the optimization even more challenging.  To address these issues, we propose a weighted greedy algorithm to sequentially determine each $(m_j, \x_j)$ pair.  

\subsection{Weighted Greedy Optimization}\label{sect:greedy}

Specifically, at each step, we maximize the mutual information increment on one pair of input and fidelity, weighted by the corresponding cost. Suppose at step $k$, we have identified a set of $k$ inputs and fidelities at which to query, $\Qcal_{k} = \{(\x_j, m_j)\}_{j=1}^{k}$. To identify a new pair of input and fidelity at step $k+1$, we maximize a weighted incremental version of \eqref{eq:ac-batch2}, 
\begin{align}
	\hat{a}_{k+1}(\x, m)& =\frac{1}{A}\sum_{l=1}^A\frac{\bbI\left(\Ycal_k \cup \{\y_m(\x)\}, \y_M(\x'_l)|\Dcal\right) - \bbI\left(\Ycal_k, \y_M(\x'_l)|\Dcal\right)}{\lambda_m} \notag \\
	 &\text{s.t.}\;\; {\lambda_{m} + \sum_{\widehat{m} \in \Mcal_k} \lambda_{\widehat{m}}   } \le B, \label{eq:ac-greedy}
\end{align}
where $\Ycal_k = \{\y_{m_j}(\x_j)|(\x_j, m_j) \in \Qcal_k \}$, and $\Mcal_k$ are all the fidelities in $\Qcal_{k}$. At the beginning ($k=0$), we have $\Qcal_{k} = \emptyset$, $\Ycal_k = \emptyset$ and $\Mcal_k = \emptyset$. Each step, we look at each valid fidelity, \ie $1 \le \lambda_m \le B - \sum_{\widehat{m} \in \Mcal_k}\lambda_{\widehat{m}}$, and find the optimal input. We then add the pair $(m, \x)$ that gives the largest $\hat{a}_{k+1}$ into $\Qcal_k$, and proceed to the next step. Our greedy approach is summarized in Algorithm \ref{algo:algo}. 

The intuition of our approach is as follows. Mutual information is a classic submodular function~\citep{krause2005near}, and hence if there were no budget constraints or weights, greedily choosing the input which increases the mutual information most achieves a solution within $(1-1/e)$ of the optimal~\citep{krause2014submodular}.  However, \citet{leskovec2007cost} observed that when items have a weight (corresponding to the cost for different fidelities in our case) and there is a budget constraint, then the approximation factor can be arbitrarily bad. 
We observe, {\it and prove}, that this only occurs as the budget is about to be filled, and up until that point, the weighted greedy optimization achieves the best possible $(1-1/e)$-approximation of the optimal.  We can formalize this {\it near $(1-1/e)$-approximation} in two ways, proven in the Appendix.  Let OPT($B$) be the maximum amount of mutual information that can be achieved with a budget $B$.  

\begin{theorem}
\label{thm:near-sm-opt}
At any step of Weighted-Greedy (Algorithm \ref{algo:algo}) before any choice of fidelity would exceed the budget, and the total budget used to that point is $B' < B$, then the mutual information of the current solution is within $(1-1/e)$ of OPT($B'$).  
\end{theorem}

\begin{corollary}
\label{cor:sm-opt+}
If Weighted-Greedy (Algorithm \ref{algo:algo}) is run until input-fidelity pair $(\x,m)$ that corresponds with the maximal acquisition function $\hat a_{k+1}(\x,m)$ would exceed the budget, it selects that input-fidelity pair anyways (the solution exceeds the budget $B$) and then terminates, the solution obtained is within $(1-1/e)$ of OPT($B$).  
\end{corollary}

We next sketch the main idea of how to prove the main result.  
Adding a new fidelity-input pair $(m, \x)$ gives an increment of learning benefit $\Delta_j = \bbI\left(\Ycal_k \cup \{\y_m(\x)\}, \y_M(\x'_l)|\Dcal\right) - \bbI\left(\Ycal_k, \y_M(\x'_l)|\Dcal\right)$. Since we need to trade off to the cost $\lambda_m$, we can view there are ${\lambda_m}$ independent copies of $\x$ (for the particular fidelity $m$), and adding each copy gives $\frac{\Delta_j}{\lambda_{m}}$ benefit. We choose the optimal $\x$ and $m$ that maximizes the benefit $\frac{\Delta_j}{\lambda_{m}}$. Since all the $\lambda_m$ copies of $\x$ have the equal, biggest benefit (among all possible choices of inputs in $\Omega$ and fidelities in $\Mcal$), we choose to add them first (greedy strategy), which in total gives $\Delta_j$ benefit -- their benefit does not diminish as each copy is added.  Via dividing by $\lambda_m$ and considering the copies, which each have unit weight, we can apply the standard submodular optimization analysis obtaining $(1-1/e)$OPT, at least until we encounter the budget constraint.


\subsection{Algorithm Complexity}
 The time complexity of our weighted greedy optimization is $\Ocal(\frac{B}{\lambda_1} MG)$ where $\lambda_1$ is the cost for the lowest fidelity, $G$ is the cost of the underlying numerical optimization (\eg L-BFGS) and acquisition function computation (detailed in \citep{li2020deep}). 
  The space complexity is $\Ocal(\frac{B}{\lambda_1}(r+d+1))$, which is to store at most $B/\lambda_1$ pairs of input locations and fidelities, and their corresponding outputs (\ie the querying results). Therefore, both the time and space complexities are linear to the budget $B$.

\begin{algorithm}[t]
	\small 
	\caption{Weighted-Greedy( $\{\lambda_m\}$, budget $B$)}          
	\label{alg:greedy}                           
	\begin{algorithmic}[1]                    
		\STATE $k \leftarrow 0, \Qcal_k \leftarrow \emptyset,  C_k \leftarrow 0$
		\WHILE {$C_k  \le B$}
		\STATE Optimize the weighted incremental acquisition function in \eqref{eq:ac-greedy}:  
		\[
		(\x^*, m^*) = \argmax_{\x \in \Omega, 1\le m \le B -C_k } \hat{a}_{k+1}(\x, m)
		\] 
		 \IF {Infeasible}
		 \STATE break 
		 \ENDIF
		 \STATE $k \leftarrow k + 1$
		\STATE $\Qcal_k \leftarrow \Qcal_{k-1} \cup \{(\x^*, m^*)\}$
		\STATE  $C_k \leftarrow C_{k-1} + \lambda_{m^*}$
		\ENDWHILE
		\STATE Return  $\Qcal_k$
	\end{algorithmic}\label{algo:algo}
\end{algorithm}


\section{Related Work}\label{sect:rel}
As an important branch of machine learning, active learning has been studied for a long time~\citep{balcan2007margin,settles2009active,balcan2009agnostic,dasgupta2011two,hanneke2014theory}. While many prior works develop active learning algorithms for kernel based models, \eg \citep{schohn2000less,tong2001support,joshi2009multi,krause2008near,li2013adaptive,huang2010active}, recent research focus has transited to  deep neural networks. \citet{gal2017deep} used Monte-Carlo (MC) Dropout~\citep{gal2016dropout} to generate posterior samples of the neural network output, based on which a variety of acquisition functions, such as  predictive entropy and  Bayesian Active Learning by Disagreement (BALD)~\citep{houlsby2011bayesian}, can be calculated and optimized  to query new examples in active learning. This approach has been proven successful in image classification. \citet{kirsch2019batchbald} developed BatchBALD,  a greedy approach that incrementally selects a set of unlabeled images under the BALD principle and issues batch queries to improve the active learning efficiency. They show that the batch acquisition function based on BALD is submodular, and hence the greedy approach achieves a $1-1/e$ approximation. Other  works along this line includes  \citep{geifman2017deep,sener2018active} based on core-set search, \citep{gissin2019discriminative,ducoffe2018adversarial} based on adversarial networks or samples, \citep{ash2019deep} based on the uncertainty evaluated in the gradient magnitude, \etc

Different from most methods, DMFAL~\citep{li2020deep} conducts optimization-based, rather than pool-based active learning. That is, they optimize the acquisition function in the entire domain rather than rank a set of pre-collected examples to label. It is also related to Bayesian experimental design~\citep{kleinegesse2020bayesian}. In addition, DMFAL for the first time automatically queries examples at different fidelities, and integrates these examples in a deep multi-fidelity model to improve the active learning efficiency while reducing the cost --- which is critical in physical simulation and engineering design.  The pioneer work of \citet{settles2008active}
empirically studies how the human labeling cost can vary in practical active learning, but does not provide a scheme to  identify multi-fidelity queries. Our work is an extension of  \citep{li2020deep} to generate a batch of multi-fidelity queries so as to reduce the query correlations, improve the diversity and quality of the training examples, while respecting a given budget constraint. A counterpart in the Bayesian optimization domain is the recent work BMBO-DARN~\citep{li2021batch}, which considers batch multi-fidelity queries  for optimizing a black-box function. From the high-level view, BMBO-DARN and our work employ a similar interleaving procedure,  \ie determining a new batch of queries by optimizing an acquisition function, issuing the queries to collect new examples, and updating the surrogate model. However, both the learning setting and acquisition function are different. More important,  we consider the budget constraint while BMBO-DARN does not. Thereby, the computation and optimization techniques of the two works are totally different. The BMBO-DARN uses Hamiltonian Monte-Carlo samples of the single  function output prediction and constructs empirical covariance matrices to compute the acquisition function, while our method and~\citep{li2020deep} use the multi-variate delta method and matrix identities to overcome the challenge of the high-dimensional outputs. BMBO-DARN uses alternating optimization to search over multiple fidelities, while our work develops a weighted greedy algorithm with additional theoretical guarantees to respect the budget constraint.


\section{Experiment}
\subsection{Solving Partial Differential Equations}\label{sect:pde}
We first evaluated \ours in several benchmark tasks of computational physics, \ie predicting the solution fields of  partial differential equations (PDEs), including \textit{Heat}, \textit{Poisson}'s, and  \textit{Burgers}'   equations~\citep{olsen2011numerical}. A numerical solver was used to collect the training examples. High-fidelity examples were generated by running the solver with dense meshes, while low-fidelity examples by coarse meshes.  The dimension of the output is the number of the grid points. For example, a $50 \times 50$ mesh means the output dimension is  $2,500$. We provided two-fidelity queries for each PDE, corresponding to $16 \times 16$ and $32 \times 32$ meshes. We also tested three-fidelity queries for Poisson's equation, with a $64 \times 64$ mesh to generate examples at the third fidelity. We denote the three-fidelity setting by Poisson-3. The input comprises of the PDE parameters and/or boundary/initial condition parameters. The details are provided in~\citep{li2020deep}. We ran the solver at each fidelity for many times, and computed the average running time. We normalized the average running time to obtain the querying cost at each fidelity, $\lambda_1 = 1$, $\lambda_2 = 3$ and $\lambda_3 = 10$. 
To collect the initial training dataset for active learning, we generated $10$ fidelity-1 examples and $2$ fidelity-2 examples for in the two-fidelity setting, and $10$, $5$, and $2$ examples of fidelity-1, 2, 3, respectively, for the three-fidelity setting. 
The training inputs of the initial dataset were uniformly sampled from the domain. To evaluate the prediction accuracy, for each PDE, we randomly sampled $500$ inputs, calculated the ground-truth outputs from a much denser mesh: $128 \times 128$ for Burger's and Poisson's and $100 \times 100$ for Heat equation. The predictions at the highest fidelity were then interpolated to these denser meshes~\citep{zienkiewicz1977finite} to evaluate the error. 

\noindent\textbf{Competing Methods.} Since currently there is not any budget-aware, batch multi-fidelity active learning approach (to our knowledge), for comparison, we first made a simple extension of the state-of-the-art multi-fidelity active learning method, \dmfal~\citep{li2020deep}. Specifically, to obtain a batch of queries, we ran \dmfal as it is, namely, each step acquiring one example  by maximizing \eqref{eq:ac-single} and then retraining the model, until the budget is exhausted or no new queries can be issued (otherwise the budget will be exceeded).  We denote this method by (i) \texttt{\dmfalbc}. Note that it is still sequentially querying and training inside each batch, but respects the budget. Next, to confirm the effectiveness of the proposed new acquisition function  ~\eqref{eq:new-acqn} (based on which, we propose our batch acquisition function~\eqref{eq:ac-batch}), we ran  active learning in the same away as \texttt{\dmfalbc}, except the acquisition function is replaced by $\frac{{a}_s(m, \x)}{\lambda_{m}}$ where $a_s$ is defined by \eqref{eq:new-acqn}. To compute $a_s$, we used an Monte-Carlo approximation similar to~\eqref{eq:ac-batch2}, where the number of samples $A$ is $20$.  We denote this method by (ii) \texttt{MFAL-BC}. In addition, we compared with (iii) \texttt{\dmfalbc-RF}, which follows the running of \texttt{\dmfalbc}, but each time, it randomly selects a fidelity $m$, then maximize the mutual information $\bbI\big(\y_m(\x), \y_M(\x)|\Dcal\big)$ ---  the numerator of \eqref{eq:ac-single} --- to identify the corresponding input.  Similarly, we tested (iv) \texttt{MFAL-BC-RF}, which follows the execution of \texttt{MFAL-BC}, but each time, it randomly selects a fidelity $m$ and maximizes $a_s(m, \x)$. Again $a_s$ is computed by the Monte-Carlo approximation. For all these methods, we used L-BFGS for the input optimization. Finally, we tested (v) \texttt{Batch-FR-BC}, which randomly samples both the fidelity and input at each step (\ie fully random), until the budget is used up or no more fidelities are available. Then the  batch of the acquired examples were added into the dataset altogether to retrain the model. Note that there are other possible acquisition functions, \eg the popular predictive variances and BALD~\citep{houlsby2011bayesian}. Their multi-fidelity versions have been tested and compared against \dmfal in ~\citep{li2020deep}, and turns out to be inferior. Hence, we did not test their possible variants/extensions in our paper. 
\begin{figure*}[t]
	\centering
	\setlength\tabcolsep{0pt}
	\includegraphics[width=0.7\textwidth]{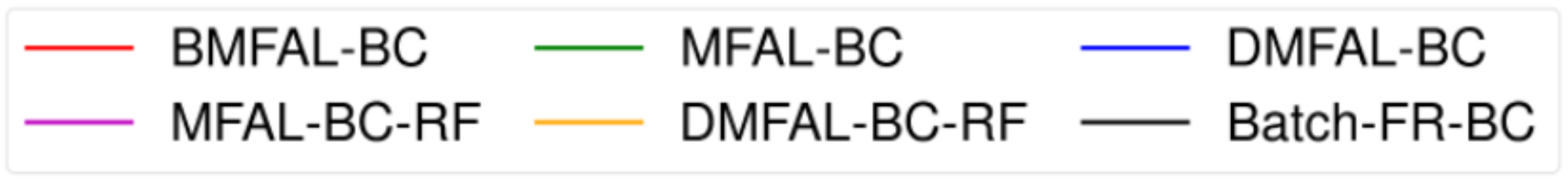}
	\begin{tabular}[c]{cc}
		\setcounter{subfigure}{0}
		\begin{subfigure}[t]{0.5\textwidth}
			\centering
			\includegraphics[width=\textwidth]{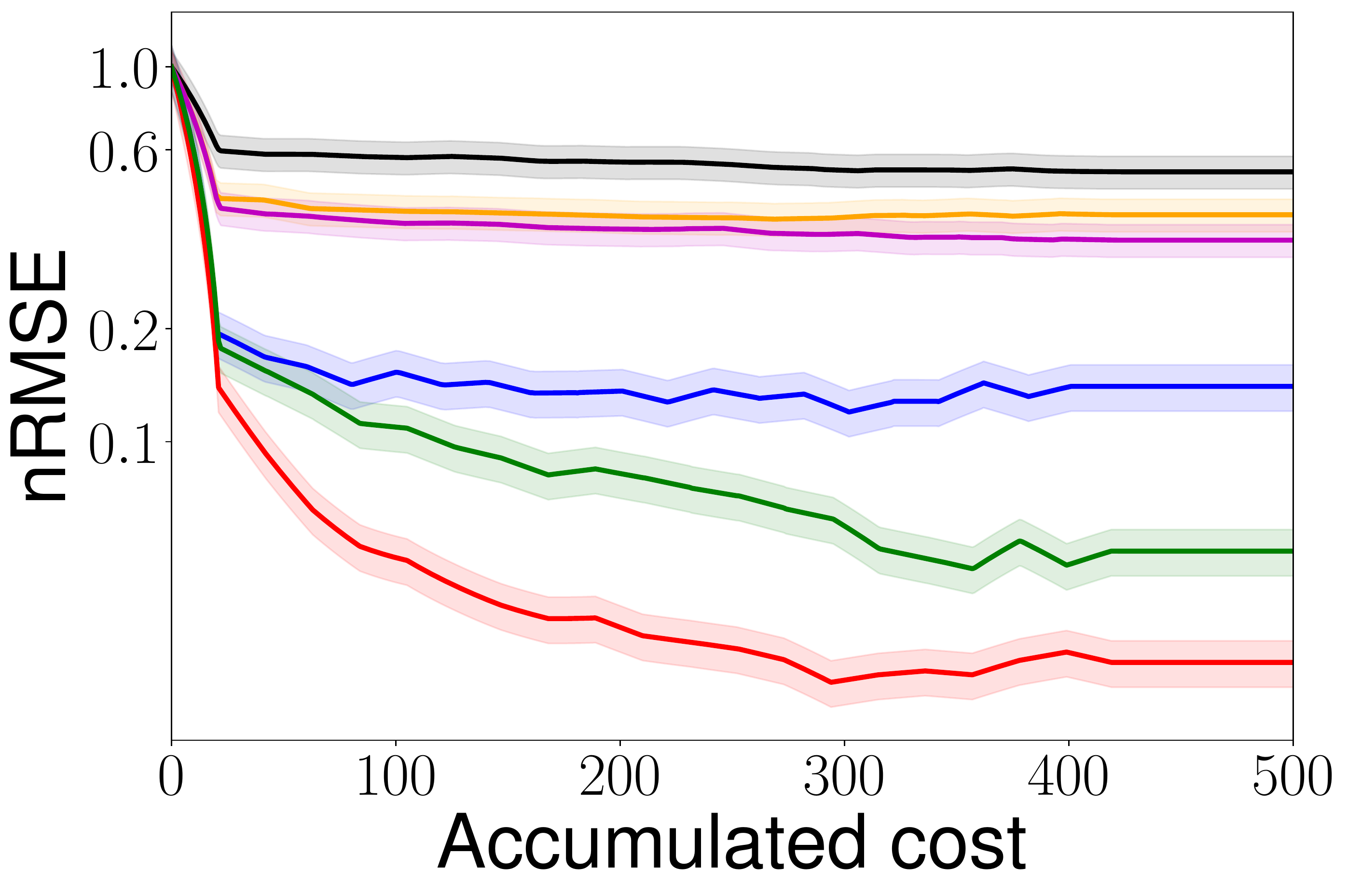}
			\caption{\small \textit{Heat}}
		\end{subfigure}
		&
		\begin{subfigure}[t]{0.5\textwidth}
			\centering
			\includegraphics[width=\textwidth]{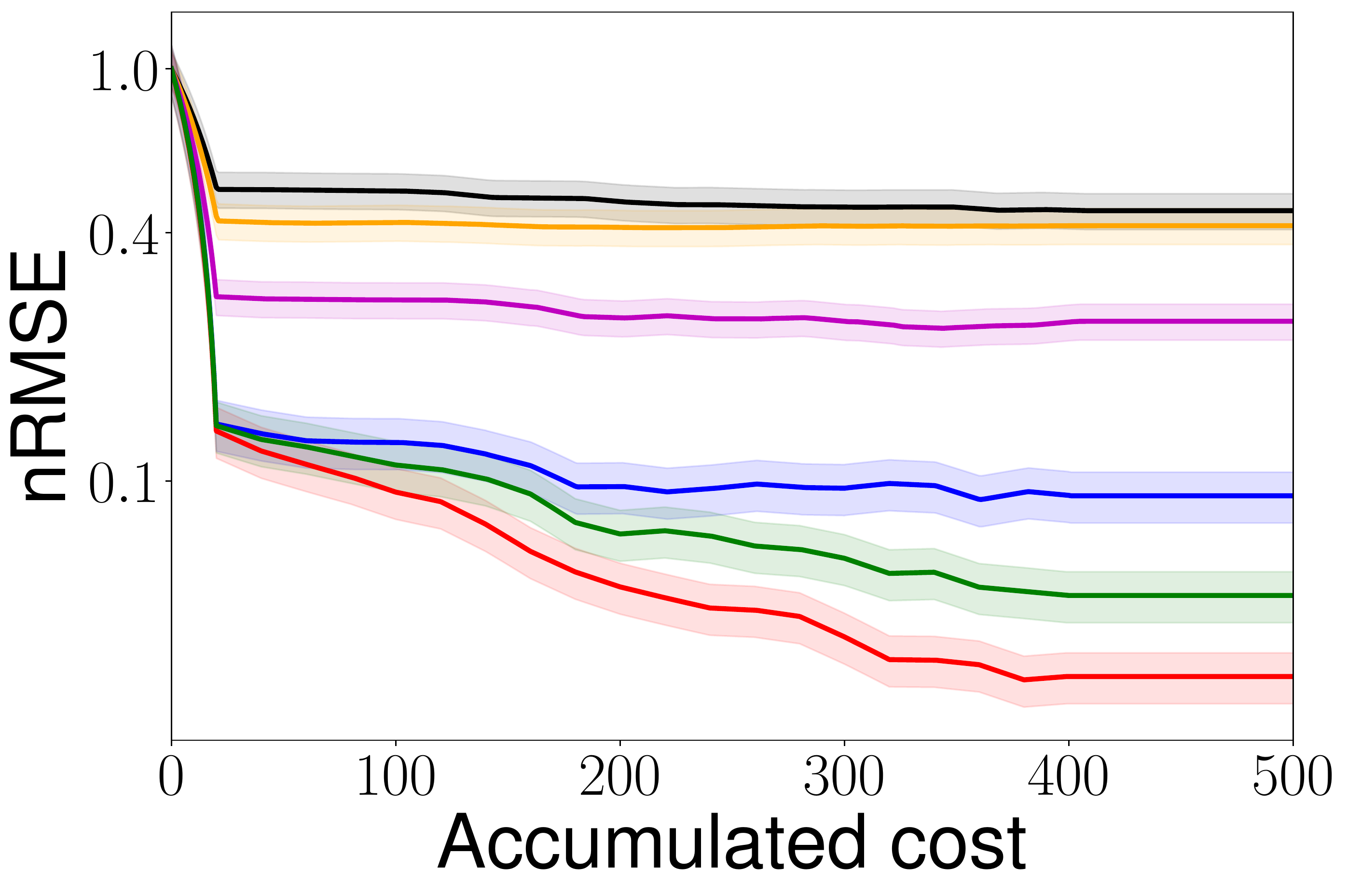}
			\caption{\small \textit{Poisson}}
		\end{subfigure}
		\\
		\begin{subfigure}[t]{0.5\textwidth}
			\centering
			\includegraphics[width=\textwidth]{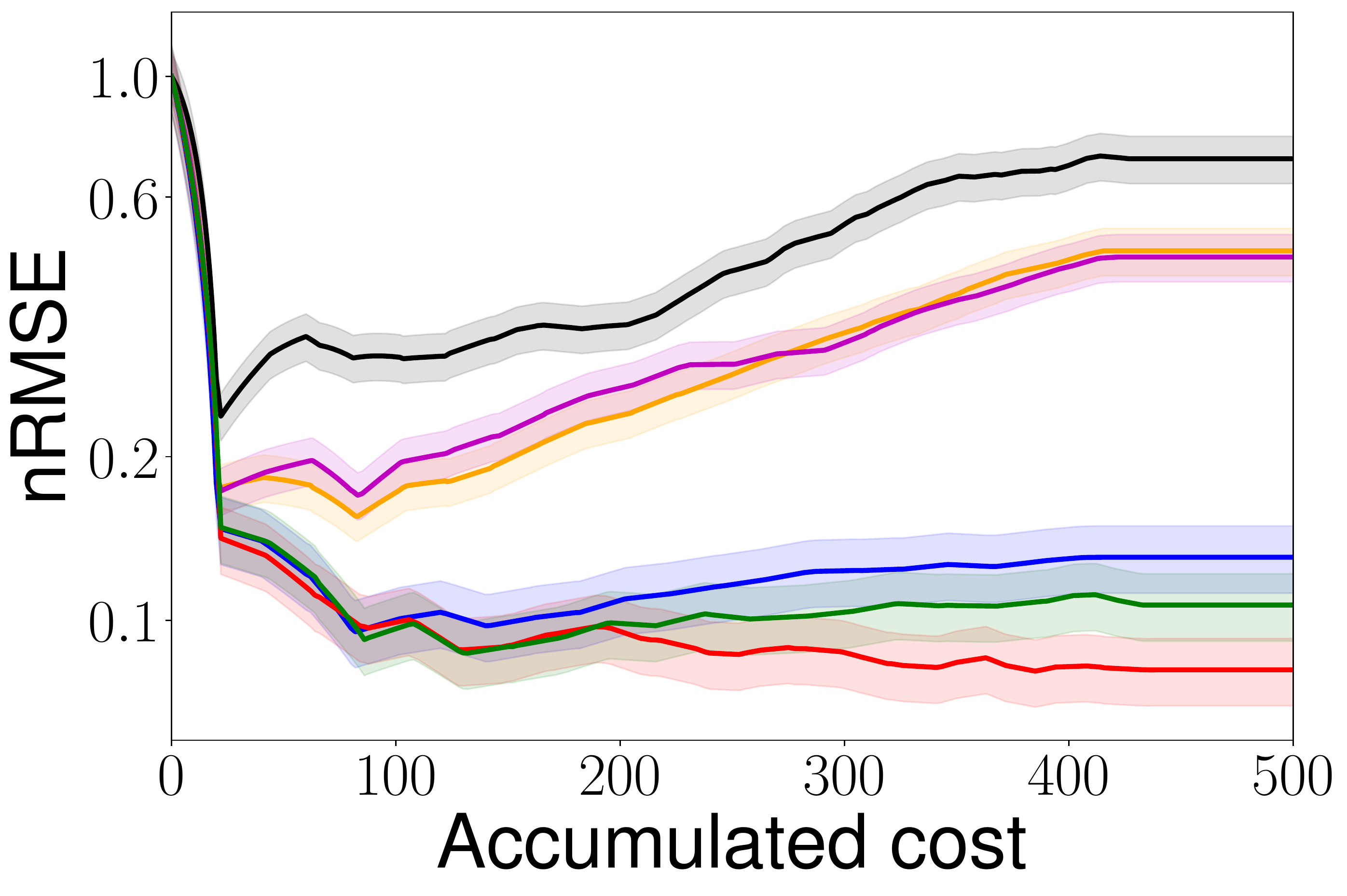}
			\caption{\small \textit{Burgers}}
		\end{subfigure}
		&
		\begin{subfigure}[t]{0.5\textwidth}
		\centering
		\includegraphics[width=\textwidth]{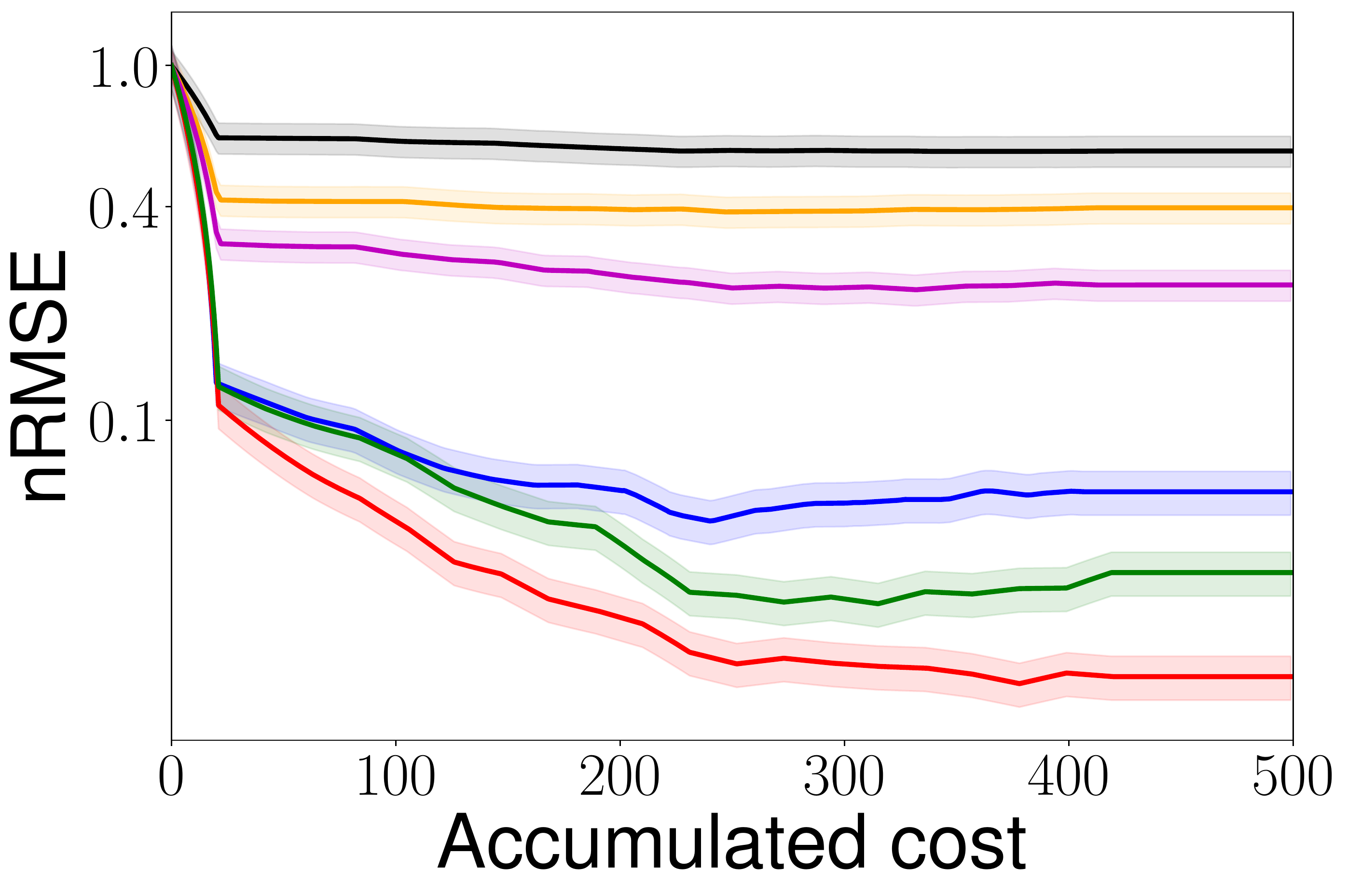}
		\caption{\small \textit{Poisson-3}}
	\end{subfigure}
	\end{tabular}
	\caption{\small Normalized root-mean-square error (nRMSE) \textit{vs.} accumulated data acquiring cost during batch multi-fidelity active learning, with budget $20$ (normalized seconds) per batch.  There are two fidelities in acquiring the examples in (a-c) and three fidelities in (d). The results were averaged over $5$ runs. The shaded region shows the standard deviation.} \label{fig:solving-pde-2fid}
\end{figure*}

\noindent\textbf{Settings and Results.} All the methods were implemented by Pytorch~\citep{paszke2019pytorch}. We followed the same setting as in \citet{li2020deep} to train the deep multi-fidelity model (see Sec. \ref{sect:model}), which employed a two-layer NN at each fidelity, \texttt{tanh} activation,  and the layer width was selected from $\{20, 40, 60, 80, 100\}$ from the initial training data. The dimension of the latent output was 20. The learning rate was tuned from $\{10^{-4}, 5\times 10^{-4}, 10^{-3}, 5 \times 10^{-3}, 10^{-2}\}$.
We set the budget for acquiring each batch to $20$ (normalized seconds), and ran each method to acquire $25$ batches of training examples. We tracked the running performance of each method in terms of the normalized root-mean-square-error (nRMSE). We repeated the experiment for five times, and report the average nRMSE \textit{vs.} the accumulated cost (\ie the summation of the corresponding $\lambda$'s in each acquired example) in Fig. \ref{fig:solving-pde-2fid}.  The shaded region exhibits the standard deviation. As we can see, \ours consistently outperforms all the competing methods by a large margin. At very early stages, all the methods exhibit similar prediction accuracy. This is reasonable, because they started with the same training set. However, along with more batches of queries, \ours shows better performance, and its discrepancy with the baselines is in general growing. In addition, MFAL-BC constantly outperforms \dmfalbc. Since they conduct the same one-by-one querying and training strategy, the improvement MFAL-BC reflects the advantage of our new single acquisition function~\eqref{eq:new-acqn} over the one used in \dmfal. Together these results have demonstrated the advantage of our method. 

\cmt{
\begin{figure*}[t]
	\centering
	\setlength\tabcolsep{0pt}
	\includegraphics[width=0.7\textwidth]{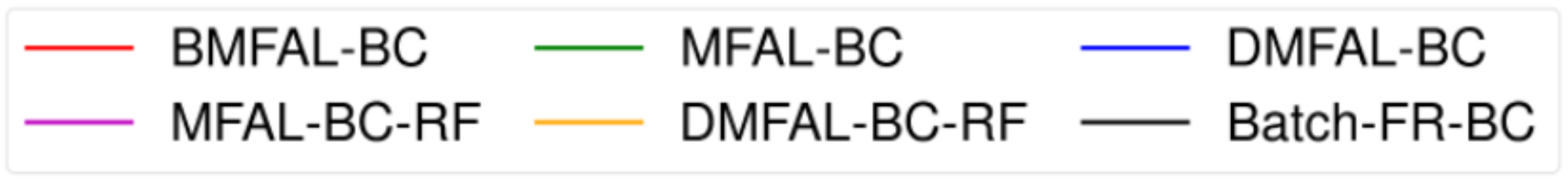}
	\begin{tabular}[c]{ccc}
		\setcounter{subfigure}{0}
		\begin{subfigure}[t]{0.33\textwidth}
			\centering
			\includegraphics[width=\textwidth]{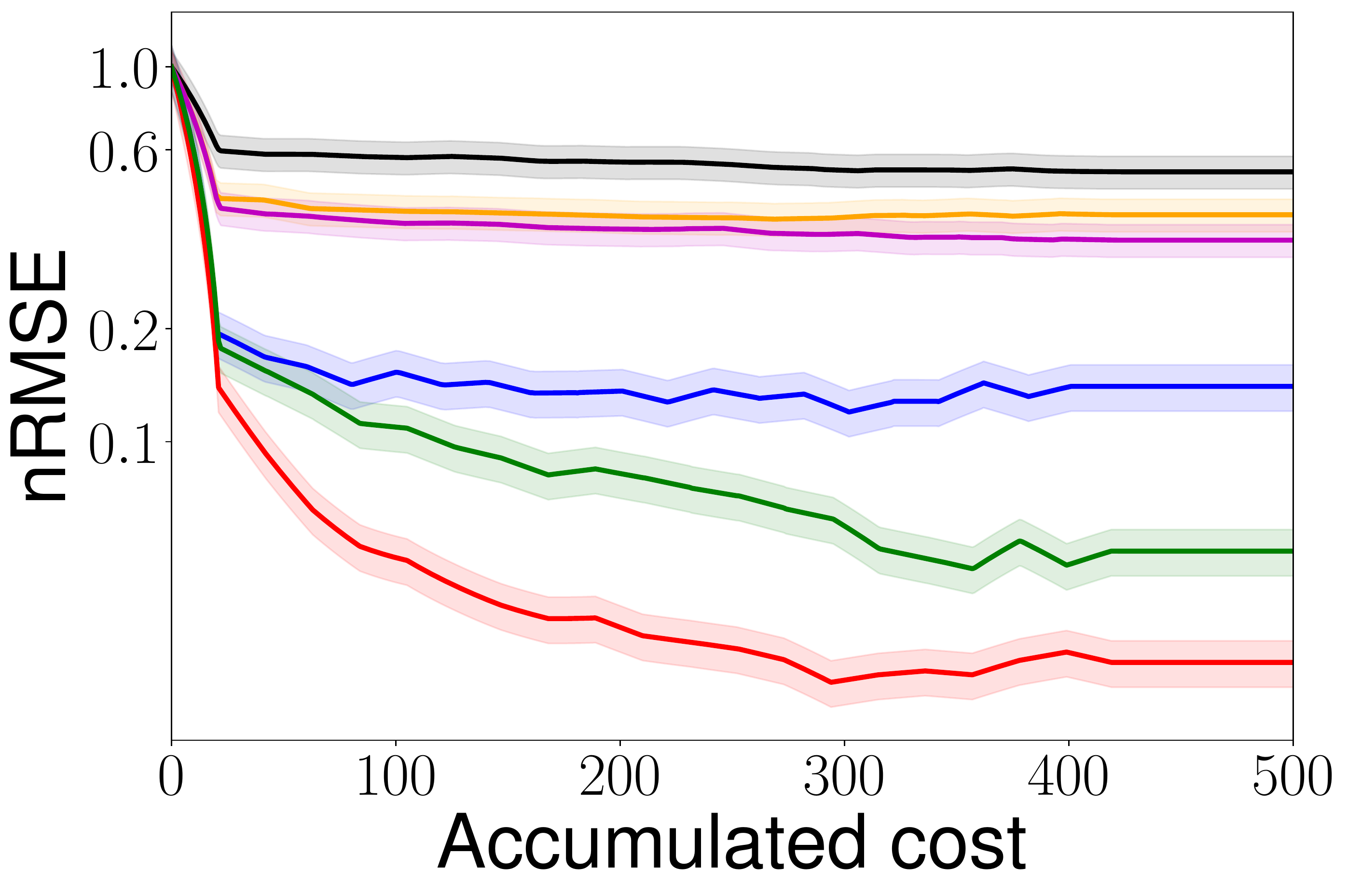}
			\caption{\small Heat equation}
		\end{subfigure}
		&
		\begin{subfigure}[t]{0.33\textwidth}
			\centering
			\includegraphics[width=\textwidth]{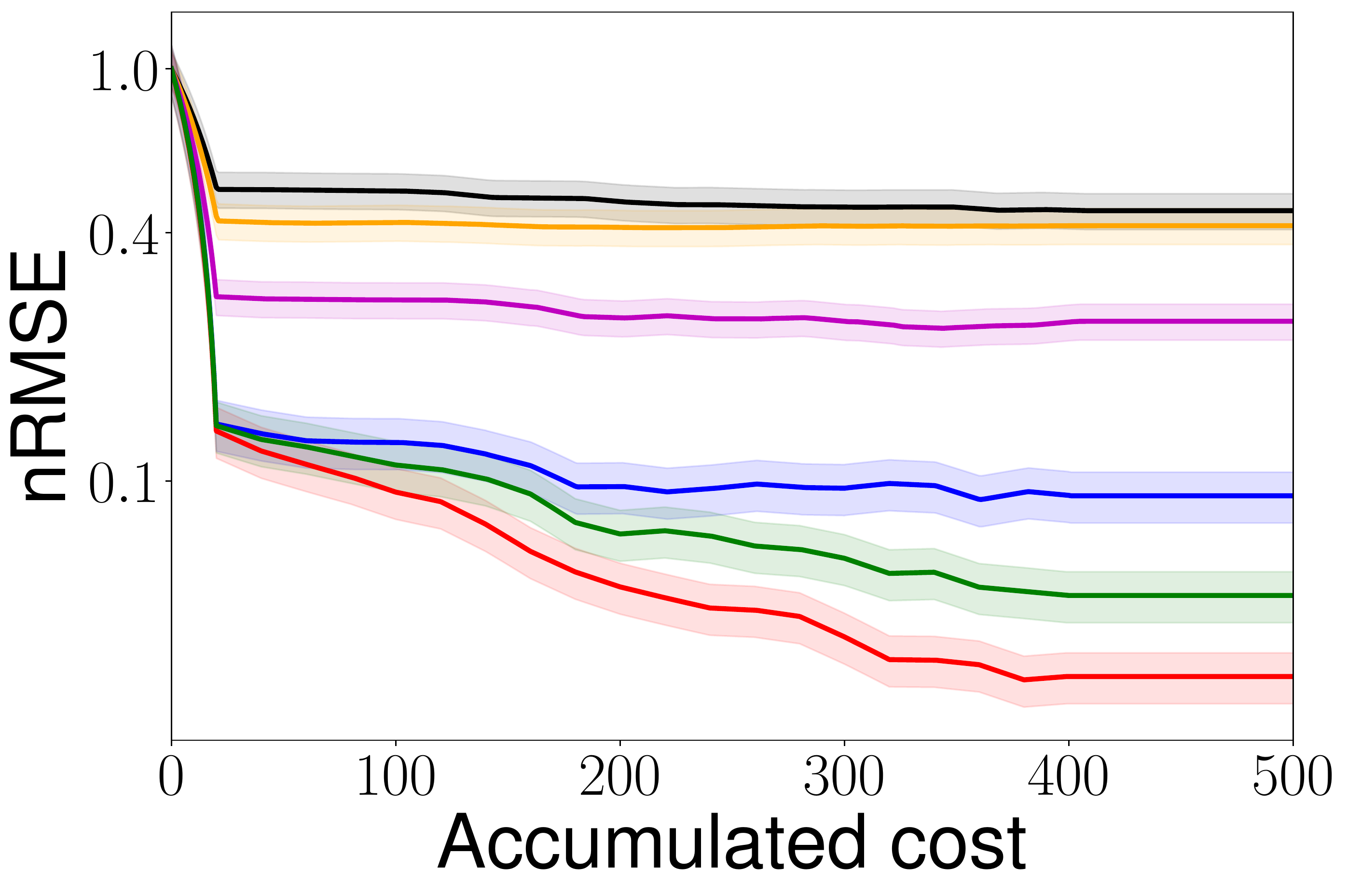}
			\caption{\small Poisson's equation}
		\end{subfigure}
		&
		\begin{subfigure}[t]{0.33\textwidth}
			\centering
			\includegraphics[width=\textwidth]{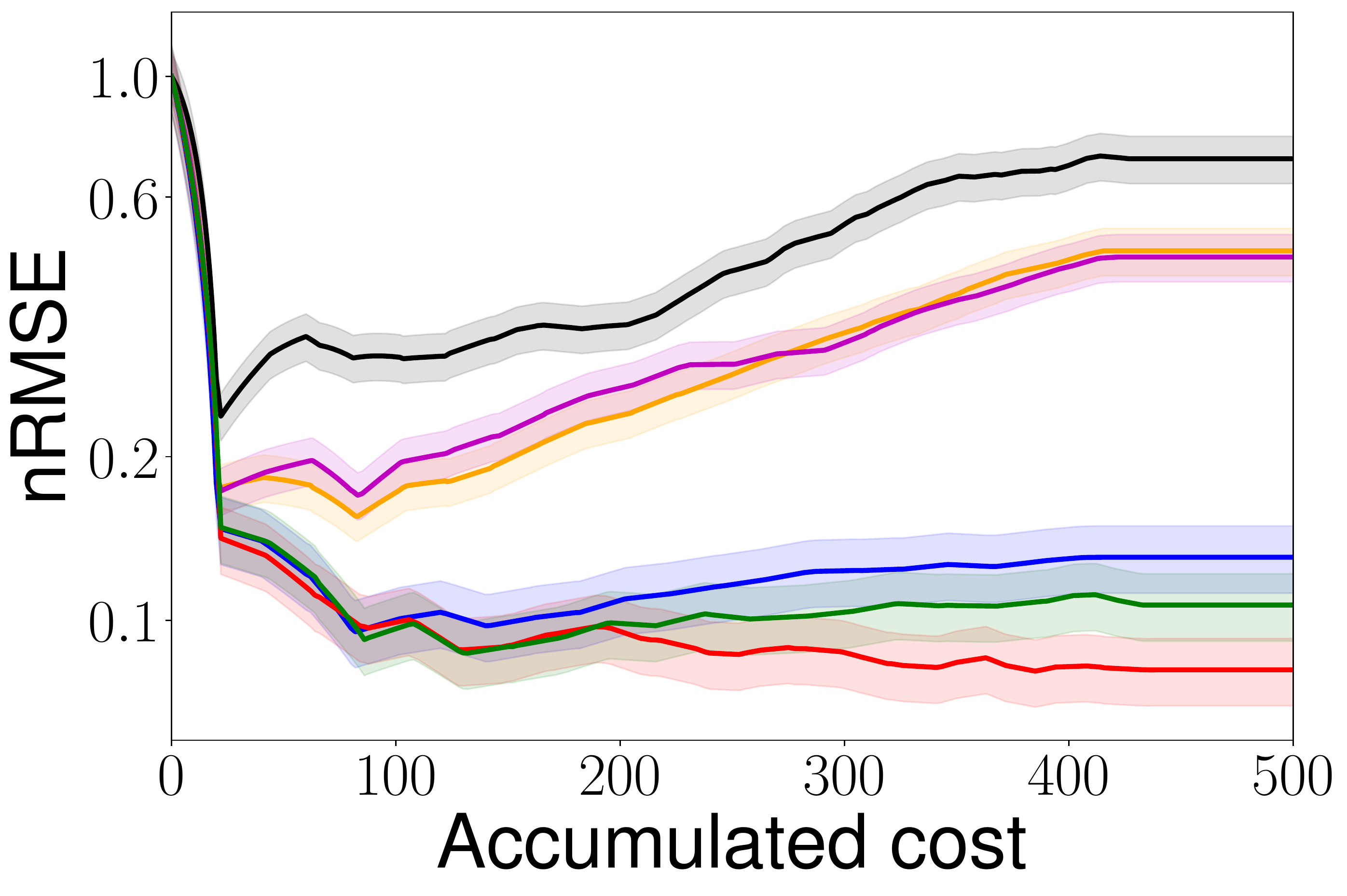}
			\caption{\small Burgers' equation}
		\end{subfigure}
	\end{tabular}
	\caption{\small Normalized root-mean-square error (nRMSE) \textit{vs.} accumulated data acquiring cost during batch multi-fidelity active learning, with budget $20$ (normalized seconds) per batch.  There are two fidelities in acquiring the examples. The results were averaged over $5$ runs. The shaded regions exhibit the standard deviation.} \label{fig:solving-pde-2fid}
\end{figure*}

\begin{figure*}[t]
	\centering
	\begin{subfigure}[t]{0.6\textwidth}
		\centering
		\includegraphics[width=\textwidth]{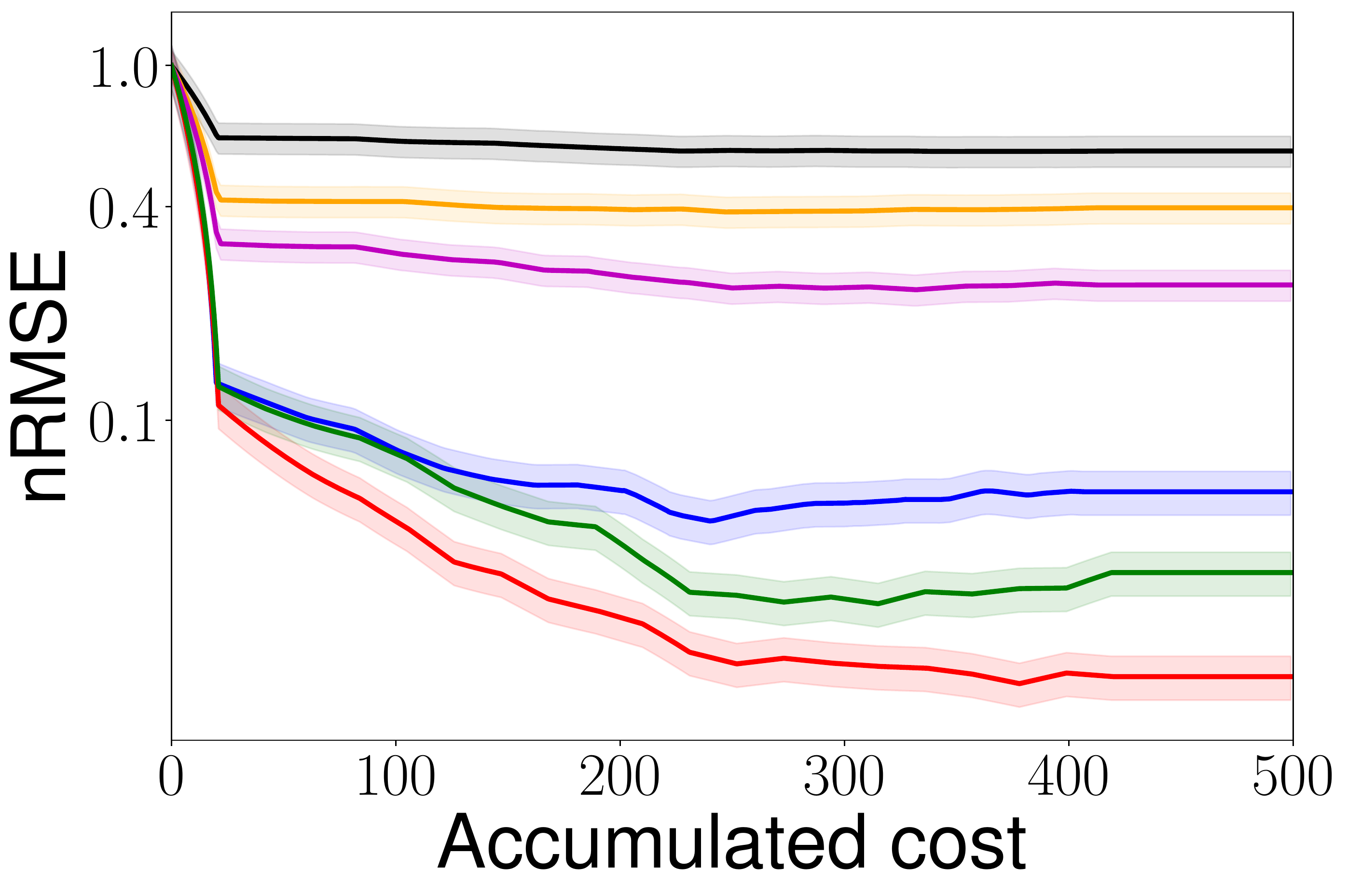}
	\end{subfigure}
\caption{\small \textit{Poisson3}} \label{fig:solving-pde-3fid}
\end{figure*}
}

\begin{figure*}[t]
	\centering
	\setlength\tabcolsep{0pt}
	\includegraphics[width=0.7\textwidth]{./figs2/legend1.pdf}
	\begin{tabular}[c]{cc}
		\setcounter{subfigure}{0}
		\begin{subfigure}[t]{0.5\textwidth}
			\centering
			\includegraphics[width=\textwidth]{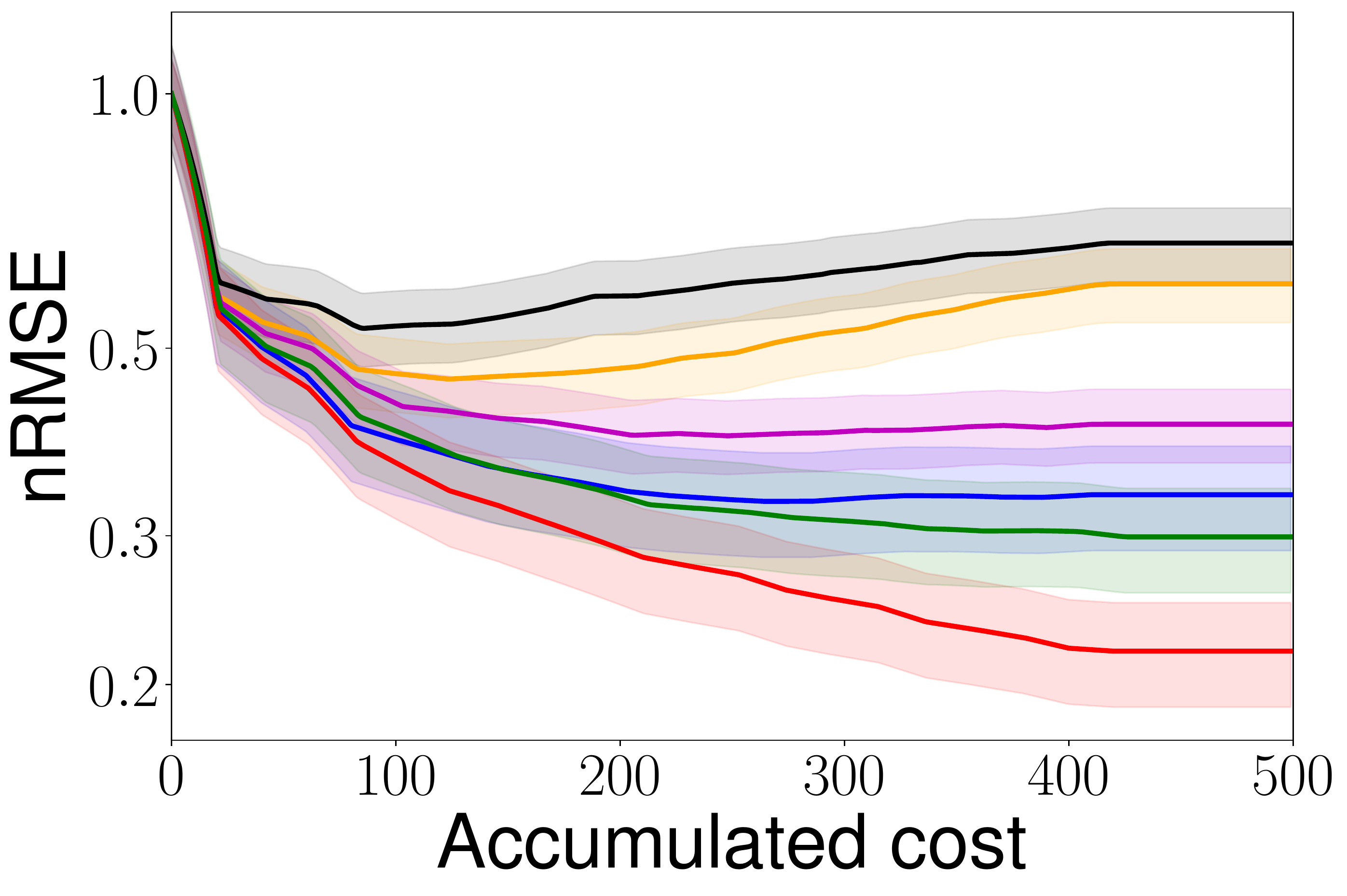}
			\caption{\small \textit{Topology optimization}} \label{fig:tpo}
		\end{subfigure}
		&
		\begin{subfigure}[t]{0.5\textwidth}
			\centering
			\includegraphics[width=\textwidth]{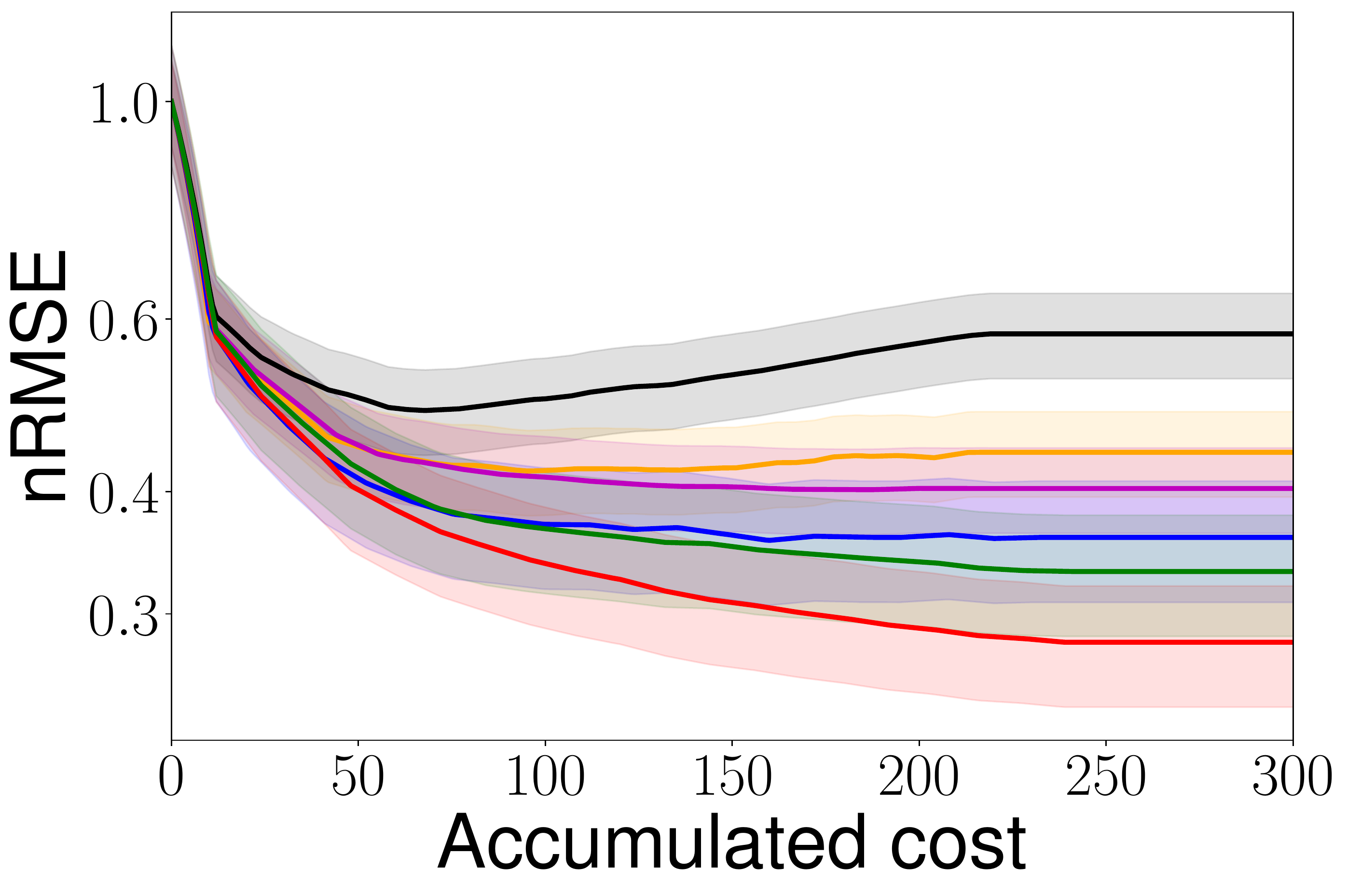}
			\caption{\small \textit{Fluid dynamics}} \label{fig:ns}
		\end{subfigure}
		
	\end{tabular}
	\caption{\small nRMSE \textit{vs.} the accumulated cost in learning to predict topological structures and fulid dynamics. The budget was set to 20 and 10 for (a) and (b), respectively.}
\end{figure*}

\subsection{Topology Structure Optimization}\label{sect:tpo}
Next, we applied our approach in topology structure optimization, which is critical in many manufacturing and engineering design problems, such as 3D printing, bridge construction, and aircraft engine production. A topology structure is used to describe how to allocate the material, \eg metal and concrete, in a designated spatial domain. Given the environmental input, \eg a pulling force or pressure, we want to identify the structure with the optimal property of interest, \eg maximum stiffness. Traditionally, we convert it to a constraint optimization problem, for which we minimize a compliance objective with a total volume constraint ~\citep{sigmund1997design}. Since the optimization often needs to repeatedly run numerical solvers to solve relevant PDEs, the computation is very expensive. We aim to learn a surrogate model with active learning, which can  predict the optimal structure outright  given the environmental input. 

Specifically, our task is to design a linear elastic structure in an L-shape domain discretized in $[0, 1] \times [0, 1]$.  The environmental input is a load at the bottom half right and described by two parameters: angle (in $[0, \frac{\pi}{2}]$) and location (in $[0.5, 1]$). Given a particular load, we aim to find the structure that has the maximum stiffness. To calculate the compliance in the structure optimization, we need to repeatedly apply the Finite Element Method (FEM)~\citep{keshavarzzadeh2018parametric}, where the fidelity is determined by the mesh. In the active learning, the training examples can be acquired at two fidelities. One corresponds to a $50 \times 50$ mesh used in the FEM, and the other $75 \times 75$. The output dimension of the two fidelities is therefore $2,500$ and $5,625$, respectively. The querying cost is the normalized average time to find the optimal structure (with the standard approach), $\lambda_1 = 1$ and $\lambda_2 = 3$. To evaluate the performance, $500$ test structures were randomly generated with a $100 \times 100$ mesh. We interpolated the high-fidelity prediction of each method to the $100 \times 100$ mesh and then evaluated the prediction error.   

At the beginning, we uniformly sampled the input (\ie load angle and location) to collect $10$ structures at the first fidelity and $2$  at the second fidelity, as the initial training set. We then ran all the active learning methods, with budget $20$ per batch, to acquire $25$ batches of examples.   We examined the average nRMSE along with the accumulated cost of acquiring the examples. The results are reported in Fig. \ref{fig:tpo}.  We can see that the prediction accuracy of \ours is consistently better than all the competing methods during the active learning. The improvement becomes larger when more examples are acquired. The results confirm the advantage of \ours.
\begin{figure*}[t]
	\centering
	\setlength\tabcolsep{0pt}
	\includegraphics[width=0.7\textwidth]{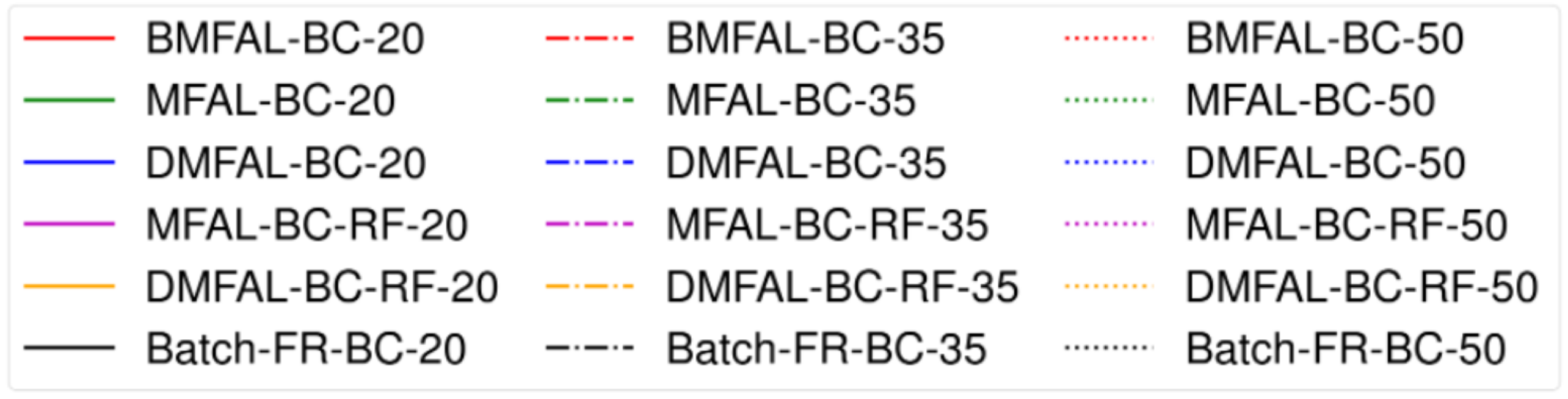}
	\begin{tabular}[c]{ccc}
		\setcounter{subfigure}{0}
		\begin{subfigure}[t]{0.33\textwidth}
			\centering
			\includegraphics[width=\textwidth]{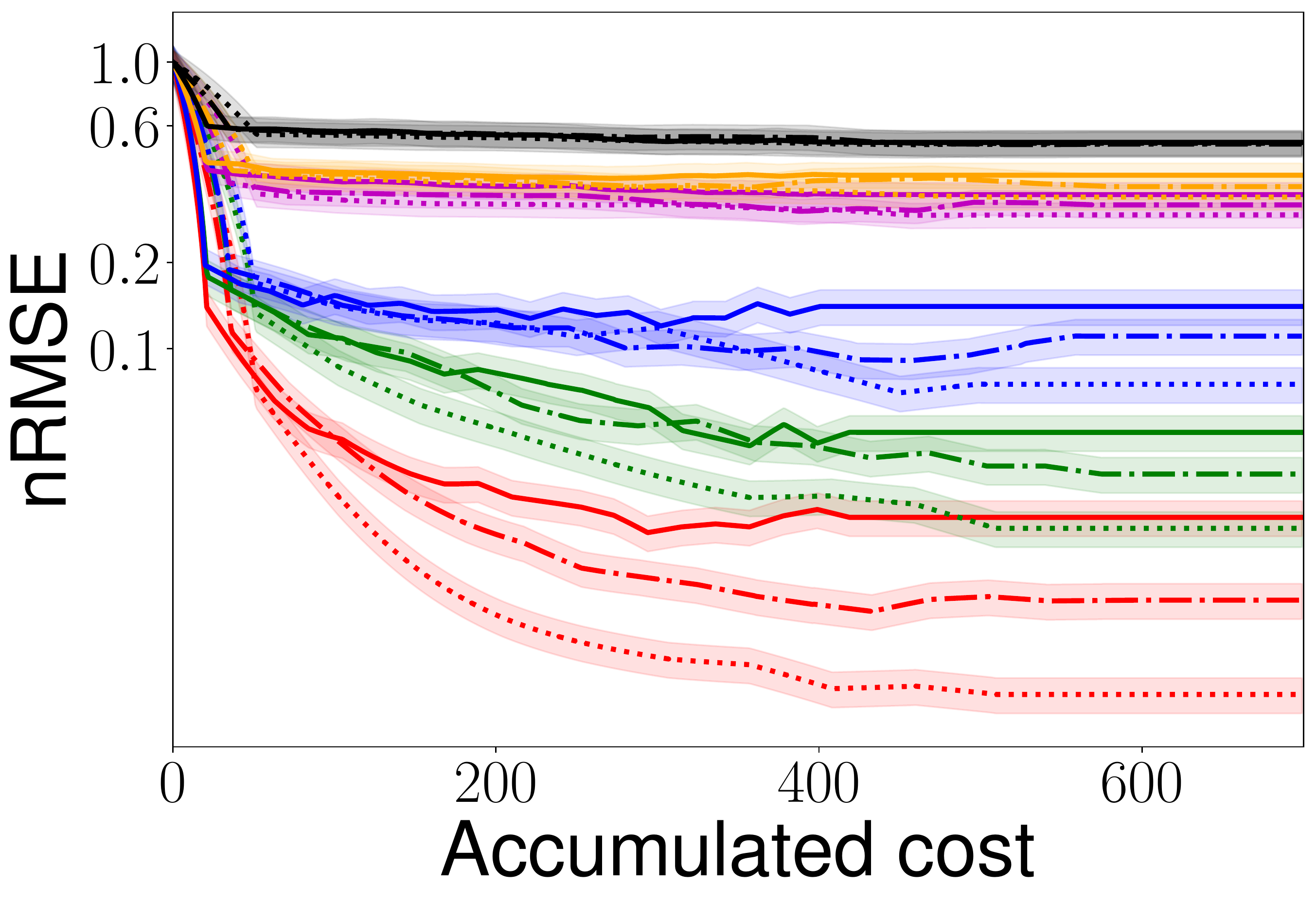}
			\caption{\small \textit{Heat}}
		\end{subfigure}
		&
		\begin{subfigure}[t]{0.33\textwidth}
			\centering
			\includegraphics[width=\textwidth]{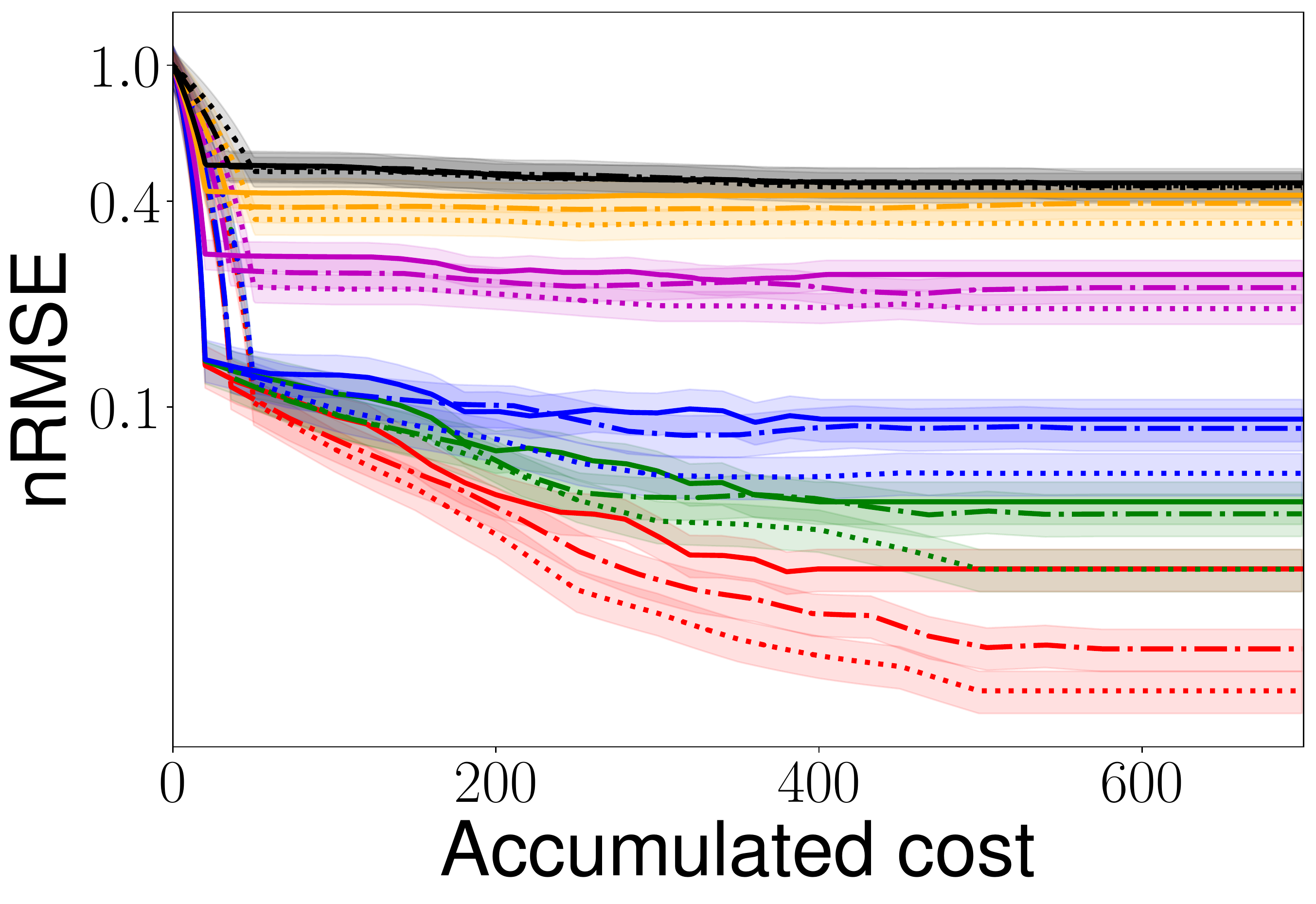}
			\caption{\small \textit{Poisson}}
		\end{subfigure}
		&
		\begin{subfigure}[t]{0.33\textwidth}
			\centering
			\includegraphics[width=\textwidth]{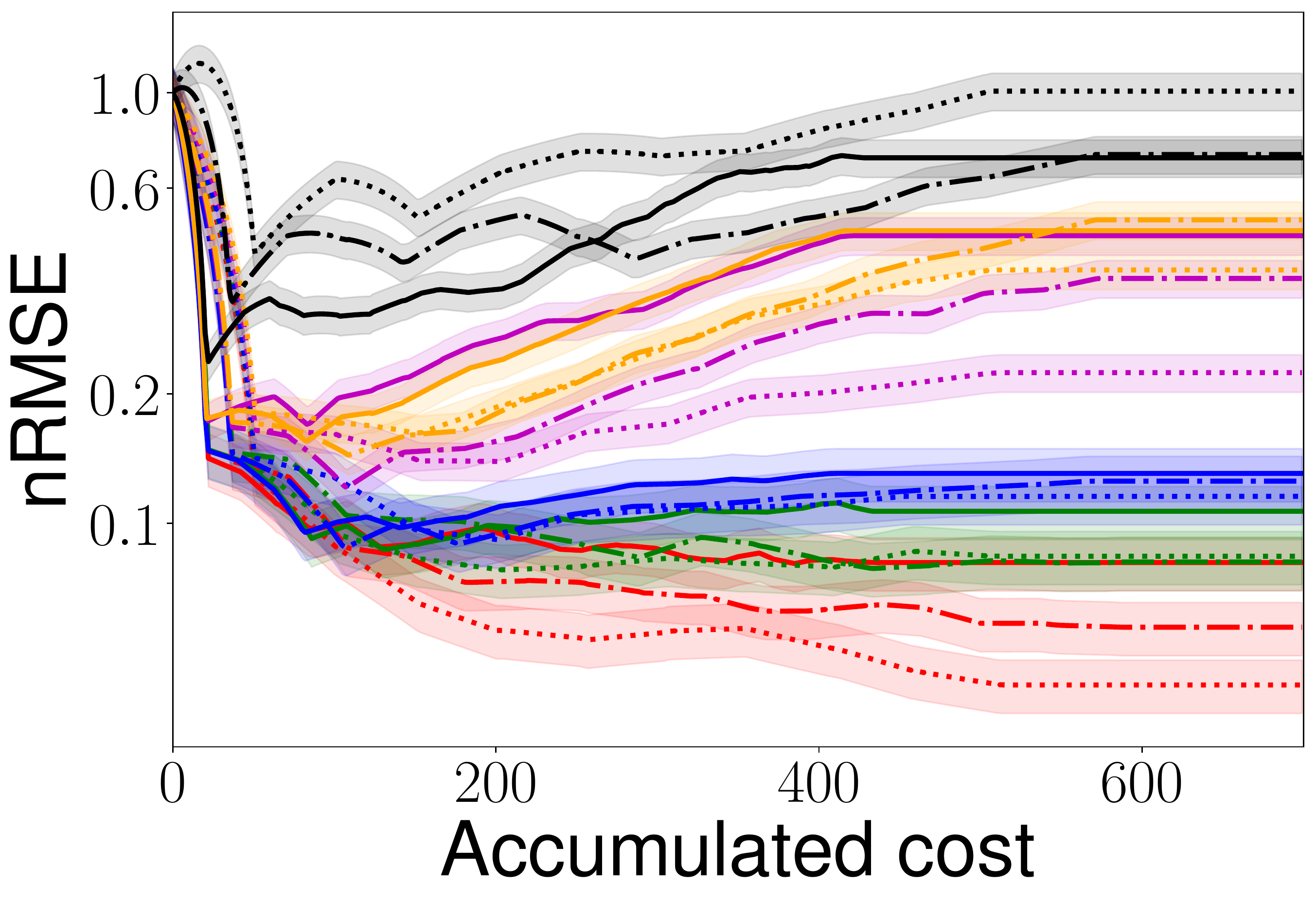}
			\caption{\small \textit{Burgers}}
		\end{subfigure}
	\end{tabular}
	\caption{\small nRMSE \textit{vs.} the accumulated cost under different budgets per batch: $B \in \{20, 35, 50\}$. } \label{fig:budget-varying}
\end{figure*}
\subsection{Predicting Spatial-Temporal Fields of Flows}\label{sect:cfd}
Third, we evaluated \ours in computational fluid dynamics (CFD).  We considered a classical application where the flow is inside a rectangular domain (represented by $[0, 1] \times [0, 1]$), and driven by rotating boundaries with a constant velocity~\citep{bozeman1973numerical}. The velocities of different parts inside the flow will vary differently, and eventually lead to turbulent fluids. To compute the velocity field in the time spatial domain, we need to solve the incompressible Navier-Stokes equations~\citep{chorin1968numerical}, which is known to be computationally challenging, especially under large Reynolds numbers. We considered the active learning task of predicting the first component of the velocity field at $20$ evenly spaced points in the temporal domain $[0, 10]$. The training examples can be queried at two fidelities. The first fidelity uses a $50 \times 50$ mesh in the $[0, 1] \times [0,1]$ spatial domain, and the second fidelity $75 \times 75$. Accordingly, the output dimensions are $50,000$ and $112,500$; the cost per query is $\lambda_1=1$ and $\lambda_2=3$.  The input is five dimensional, including the velocities of the four boundaries and the Reynold number. The details are given in~\citep{li2020deep}. For testing, we computed the solution fields of $256$ random inputs,  using a $128 \times 128$ mesh. We used the cubic-spline interpolation to align the prediction made by each method to the $128 \times 128$ grid, and then calculated normalized RMSE. To conduct the active learning, we randomly generated $10$ and $2$ examples in each fidelity as the initial training set. We set the budget to $10$, and ran each method to acquire 25 batches. 
 We repeated the experiment for five times, and examined how the average nRMSE varied along with the accumulated cost.
 As shown in Fig. \ref{fig:ns}, \ours keeps exhibiting superior predictive performance during the course of active learning. Again, the more examples acquired, the more improvement of \ours upon the competing methods. The results are consistent with the previous experiments. Note that throughout these comparisons,  we focus on the accumulated cost of querying (or generating) new examples, because it dominates the total cost, and in practice is the key bottleneck in learning the surrogate model. For example, in topology optimization (Sec \ref{sect:tpo}) and the fluid dynamic experiment, running a high-fidelity solver once takes 300-500 seconds on our hardware, while the surrogate training takes less than 2 seconds, and our weighted greedy optimization of the batch acquisition function (Algorithm \ref{algo:algo}) takes a few seconds. One can imagine for practical larger-scale problems, the simulation cost, \eg taking hours or even days to generate one example, can be even more dominant and decisive. 

\subsection{Influence of Different Budgets}
Finally, we examined how the budget choice will influence the performance of active learning. To this end, we varied the budget $B$ in $\{20, 35, 50\}$ and tested all the methods for Poisson's, Burger's and heat equations. We used the same two-fidelity settings as in Sec. \ref{sect:pde}. For each budget, we ran the experiment for five times. We show the average nRMSE \textit{vs.} the accumulated cost in Fig. \ref{fig:budget-varying}. As we can see, 
the larger the budget per batch, the better the running performance of \ours. This is reasonable, because a larger budget allows our method to generate more queries in each batch and in the meantime to account for their correlations or information redundancy.  Accordingly, the acquired training examples are overall more diverse and informative. Again, \ours outperforms all the competing methods under every budget. The improvement of \ours is bigger under larger budget choices. This together demonstrates the advantage of batch active learning that takes into account the correlations between queries.

\section{Conclusion}
We have presented \ours, a budget-aware, batch multi-fidelity active learning algorithm for high-dimensional outputs. Our weighted greedy algorithm can efficiently generate a batch of input-fidelity pairs for querying under the budget constraint, without the need for combinatorially searching over the fidelities, while achieving good approximation guarantees. The results on several typical computational physical applications are encouraging. 

\section*{Acknowledgments}
This work has been supported by MURI AFOSR grant FA9550-20-1-0358 and NSF CAREER Award IIS-2046295. JP thanks NSF CDS\&E-1953350, IIS-1816149, CCF-2115677, and Visa Research.  
\bibliographystyle{apalike}
\bibliography{BDMFAL}

\begin{thebibliography}{}

\bibitem[Ash et~al., 2019]{ash2019deep}
Ash, J.~T., Zhang, C., Krishnamurthy, A., Langford, J., and Agarwal, A. (2019).
\newblock Deep batch active learning by diverse, uncertain gradient lower
  bounds.
\newblock In {\em International Conference on Learning Representations}.

\bibitem[Balcan et~al., 2009]{balcan2009agnostic}
Balcan, M.-F., Beygelzimer, A., and Langford, J. (2009).
\newblock Agnostic active learning.
\newblock {\em Journal of Computer and System Sciences}, 75(1):78--89.

\bibitem[Balcan et~al., 2007]{balcan2007margin}
Balcan, M.-F., Broder, A., and Zhang, T. (2007).
\newblock Margin based active learning.
\newblock In {\em International Conference on Computational Learning Theory},
  pages 35--50. Springer.

\bibitem[Bickel and Doksum, 2015]{bickel2015mathematical}
Bickel, P.~J. and Doksum, K.~A. (2015).
\newblock {\em Mathematical statistics: basic ideas and selected topics, volume
  I}, volume 117.
\newblock CRC Press.

\bibitem[Bozeman and Dalton, 1973]{bozeman1973numerical}
Bozeman, J.~D. and Dalton, C. (1973).
\newblock Numerical study of viscous flow in a cavity.
\newblock {\em Journal of Computational Physics}, 12(3):348--363.

\bibitem[Chorin, 1968]{chorin1968numerical}
Chorin, A.~J. (1968).
\newblock Numerical solution of the navier-stokes equations.
\newblock {\em Mathematics of computation}, 22(104):745--762.

\bibitem[Conti and O’Hagan, 2010]{conti2010bayesian}
Conti, S. and O’Hagan, A. (2010).
\newblock Bayesian emulation of complex multi-output and dynamic computer
  models.
\newblock {\em Journal of statistical planning and inference}, 140(3):640--651.

\bibitem[Dasgupta, 2011]{dasgupta2011two}
Dasgupta, S. (2011).
\newblock Two faces of active learning.
\newblock {\em Theoretical computer science}, 412(19):1767--1781.

\bibitem[Ducoffe and Precioso, 2018]{ducoffe2018adversarial}
Ducoffe, M. and Precioso, F. (2018).
\newblock Adversarial active learning for deep networks: a margin based
  approach.
\newblock {\em arXiv preprint arXiv:1802.09841}.

\bibitem[Gal and Ghahramani, 2016]{gal2016dropout}
Gal, Y. and Ghahramani, Z. (2016).
\newblock Dropout as a {B}ayesian approximation: Representing model uncertainty
  in deep learning.
\newblock In {\em international conference on machine learning}, pages
  1050--1059.

\bibitem[Gal et~al., 2017]{gal2017deep}
Gal, Y., Islam, R., and Ghahramani, Z. (2017).
\newblock Deep bayesian active learning with image data.
\newblock In {\em International Conference on Machine Learning}, pages
  1183--1192.

\bibitem[Geifman and El-Yaniv, 2017]{geifman2017deep}
Geifman, Y. and El-Yaniv, R. (2017).
\newblock Deep active learning over the long tail.
\newblock {\em arXiv preprint arXiv:1711.00941}.

\bibitem[Gissin and Shalev-Shwartz, 2019]{gissin2019discriminative}
Gissin, D. and Shalev-Shwartz, S. (2019).
\newblock Discriminative active learning.
\newblock {\em arXiv preprint arXiv:1907.06347}.

\bibitem[Hanneke et~al., 2014]{hanneke2014theory}
Hanneke, S. et~al. (2014).
\newblock Theory of disagreement-based active learning.
\newblock {\em Foundations and Trends{\textregistered} in Machine Learning},
  7(2-3):131--309.

\bibitem[Houlsby et~al., 2011]{houlsby2011bayesian}
Houlsby, N., Husz{\'a}r, F., Ghahramani, Z., and Lengyel, M. (2011).
\newblock Bayesian active learning for classification and preference learning.
\newblock {\em arXiv preprint arXiv:1112.5745}.

\bibitem[Huang et~al., 2010]{huang2010active}
Huang, S.-J., Jin, R., and Zhou, Z.-H. (2010).
\newblock Active learning by querying informative and representative examples.
\newblock In {\em Advances in neural information processing systems}, pages
  892--900.

\bibitem[Joshi et~al., 2009]{joshi2009multi}
Joshi, A.~J., Porikli, F., and Papanikolopoulos, N. (2009).
\newblock Multi-class active learning for image classification.
\newblock In {\em 2009 IEEE Conference on Computer Vision and Pattern
  Recognition}, pages 2372--2379. IEEE.

\bibitem[Kato, 2013]{kato2013perturbation}
Kato, T. (2013).
\newblock {\em Perturbation theory for linear operators}, volume 132.
\newblock Springer Science \& Business Media.

\bibitem[Kennedy and O'Hagan, 2000]{kennedy2000predicting}
Kennedy, M.~C. and O'Hagan, A. (2000).
\newblock Predicting the output from a complex computer code when fast
  approximations are available.
\newblock {\em Biometrika}, 87(1):1--13.

\bibitem[Keshavarzzadeh et~al., 2018]{keshavarzzadeh2018parametric}
Keshavarzzadeh, V., Kirby, R.~M., and Narayan, A. (2018).
\newblock Parametric topology optimization with multi-resolution finite element
  models.
\newblock {\em arXiv preprint arXiv:1808.10367}.

\bibitem[Kingma and Welling, 2013]{kingma2013auto}
Kingma, D.~P. and Welling, M. (2013).
\newblock Auto-encoding variational bayes.
\newblock {\em arXiv preprint arXiv:1312.6114}.

\bibitem[Kirsch et~al., 2019]{kirsch2019batchbald}
Kirsch, A., van Amersfoort, J., and Gal, Y. (2019).
\newblock Batch{B}ald: Efficient and diverse batch acquisition for deep
  bayesian active learning.
\newblock In {\em Advances in Neural Information Processing Systems}, pages
  7026--7037.

\bibitem[Kleinegesse and Gutmann, 2020]{kleinegesse2020bayesian}
Kleinegesse, S. and Gutmann, M.~U. (2020).
\newblock Bayesian experimental design for implicit models by mutual
  information neural estimation.
\newblock In {\em International Conference on Machine Learning}, pages
  5316--5326. PMLR.

\bibitem[Krause and Golovin, 2014]{krause2014submodular}
Krause, A. and Golovin, D. (2014).
\newblock Submodular function maximization.
\newblock {\em Tractability}, 3:71--104.

\bibitem[Krause and Guestrin, 2005]{krause2005near}
Krause, A. and Guestrin, C.~E. (2005).
\newblock Near-optimal nonmyopic value of information in graphical models.
\newblock In {\em Proc. of Uncertainty in Artificial Intelligence (UAI)}.

\bibitem[Krause et~al., 2008]{krause2008near}
Krause, A., Singh, A., and Guestrin, C. (2008).
\newblock Near-optimal sensor placements in gaussian processes: Theory,
  efficient algorithms and empirical studies.
\newblock {\em Journal of Machine Learning Research}, 9(Feb):235--284.

\bibitem[Leskovec et~al., 2007]{leskovec2007cost}
Leskovec, J., Krause, A., Guestrin, C., Faloutsos, C., VanBriesen, J., and
  Glance, N. (2007).
\newblock Cost-effective outbreak detection in networks.
\newblock In {\em Proceedings of the 13th ACM SIGKDD international conference
  on Knowledge discovery and data mining}, pages 420--429.

\bibitem[Li et~al., 2021]{li2021batch}
Li, S., Kirby, R., and Zhe, S. (2021).
\newblock Batch multi-fidelity bayesian optimization with deep auto-regressive
  networks.
\newblock {\em Advances in Neural Information Processing Systems},
  34:25463--25475.

\bibitem[Li et~al., 2022]{li2020deep}
Li, S., Wang, Z., Kirby, R.~M., and Zhe, S. (2022).
\newblock Deep multi-fidelity active learning of high-dimensional outputs.
\newblock {\em Proceedings of the Twenty-Fifth International Conference on
  Artificial Intelligence and Statistics}.

\bibitem[Li and Guo, 2013]{li2013adaptive}
Li, X. and Guo, Y. (2013).
\newblock Adaptive active learning for image classification.
\newblock In {\em Proceedings of the IEEE Conference on Computer Vision and
  Pattern Recognition}, pages 859--866.

\bibitem[Mendes et~al., 2020]{mendes2020trimtuner}
Mendes, P., Casimiro, M., Romano, P., and Garlan, D. (2020).
\newblock Trimtuner: Efficient optimization of machine learning jobs in the
  cloud via sub-sampling.
\newblock In {\em 2020 28th International Symposium on Modeling, Analysis, and
  Simulation of Computer and Telecommunication Systems (MASCOTS)}, pages 1--8.
  IEEE.

\bibitem[Oehlert, 1992]{oehlert1992note}
Oehlert, G.~W. (1992).
\newblock A note on the delta method.
\newblock {\em The American Statistician}, 46(1):27--29.

\bibitem[Olsen-Kettle, 2011]{olsen2011numerical}
Olsen-Kettle, L. (2011).
\newblock Numerical solution of partial differential equations.
\newblock {\em Lecture notes at University of Queensland, Australia}.

\bibitem[Paszke et~al., 2019]{paszke2019pytorch}
Paszke, A., Gross, S., Massa, F., Lerer, A., Bradbury, J., Chanan, G., Killeen,
  T., Lin, Z., Gimelshein, N., Antiga, L., et~al. (2019).
\newblock Pytorch: An imperative style, high-performance deep learning library.
\newblock In {\em Advances in neural information processing systems}, pages
  8026--8037.

\bibitem[Schohn and Cohn, 2000]{schohn2000less}
Schohn, G. and Cohn, D. (2000).
\newblock Less is more: Active learning with support vector machines.
\newblock In {\em ICML}, volume~2, page~6. Citeseer.

\bibitem[Sener and Savarese, 2018]{sener2018active}
Sener, O. and Savarese, S. (2018).
\newblock Active learning for convolutional neural networks: A core-set
  approach.
\newblock In {\em International Conference on Learning Representations}.

\bibitem[Settles, 2009]{settles2009active}
Settles, B. (2009).
\newblock Active learning literature survey.
\newblock Technical report, University of Wisconsin-Madison Department of
  Computer Sciences.

\bibitem[Settles et~al., 2008]{settles2008active}
Settles, B., Craven, M., and Friedland, L. (2008).
\newblock Active learning with real annotation costs.
\newblock In {\em Proceedings of the NIPS workshop on cost-sensitive learning},
  volume~1. Vancouver, CA:.

\bibitem[Sigmund, 1997]{sigmund1997design}
Sigmund, O. (1997).
\newblock On the design of compliant mechanisms using topology optimization.
\newblock {\em Journal of Structural Mechanics}, 25(4):493--524.

\bibitem[Tong and Koller, 2001]{tong2001support}
Tong, S. and Koller, D. (2001).
\newblock Support vector machine active learning with applications to text
  classification.
\newblock {\em Journal of machine learning research}, 2(Nov):45--66.

\bibitem[Zienkiewicz et~al., 1977]{zienkiewicz1977finite}
Zienkiewicz, O.~C., Taylor, R.~L., Zienkiewicz, O.~C., and Taylor, R.~L.
  (1977).
\newblock {\em The finite element method}, volume~36.
\newblock McGraw-hill London.

\end{thebibliography}

\appendix

\section{Appendix}
\begin{figure}[H]
	\centering
	\includegraphics[width=0.5\textwidth]{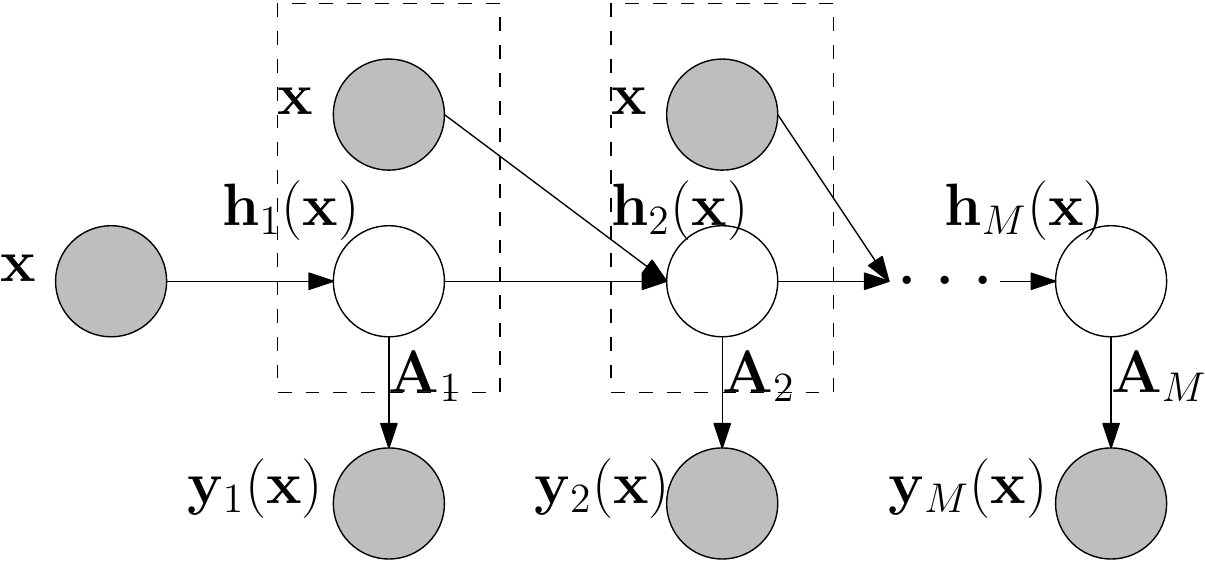}
	\caption{\small Graphical representation of DMFAL. The low dimensional latent output in each fidelity $\h_m(\x)$ ($1\le m \le M$) is generated by a (deep) neural network.} \label{fig:graphical}
\end{figure}
We now formally prove our main theoretical results on the approximate optimization properties of the Weighted-Greedy algorithm that we have proposed.  In particular, these bounds are relative to the optimal algorithm with a budget $B$, we denote its mutual information as OPT($B$).  We note that the optimal is with respect to the measurement of 
$\frac{1}{A} \sum_{l=1}^A \bbI\left(\Ycal_k, \y_M(\x'_l)|\Dcal\right)$ 
on the $A$ Monte Carlo samples, and only over the space of inputs $\Omega$ and fidelities $\Mcal$ we consider.  If more Monte Carlo samples are considered, or somehow mutual information is computed precisely, or more fidelities are searched over, then the OPT($B$) considered will increase, and the near-optimality of the greedy algorithm will continue to be approximately proportional to that optimal potential value.  Since \dmfal can actively choose an optimal $\x \in \Omega$ for a fixed fidelity $m$, which is already optimized over a continuous space, the optimal bound we consider OPT($B$) is relative to this method.  

We now restate and prove the main results.

\begin{theorem}[Theorem \ref{thm:near-sm-opt}]
At any step of Weighted-Greedy (Algorithm \ref{algo:algo}) before any choice of fidelity would exceed the budget, and the total budget used to that point is $B' < B$, then the mutual information of the current solution is within $(1-1/e)$ of OPT($B'$).  
\end{theorem}
\begin{proof}
    Given a set of elements $\tilde \Omega$ and a submodular objective function $\phi$, it is well known that if one greedily selects items from $\tilde \Omega$ that most increase $\phi$ at each step, then after $t$ steps, the selected set achieves an objective value in $\phi$ within a $(1-1/e)$-factor of the optimal set of $t$ elements from $\tilde \Omega$~\citep{krause2014submodular}.  Our objective $\frac{1}{A} \sum_{l=1}^A \bbI\left(\Ycal_k, \y_M(\x'_l)|\Dcal\right)$, where the mutual information is a classic submodular function~\citep{krause2005near}.  However, in our setting each item (an input-fidelity pair $(\x,m)$) has a cost $\lambda_m$ that counts against a total budget $B$.  Our proof will convert this setting back to the classic unweighted setting so we can invoke the standard $(1-1/e)$-result.  
    
    Our Weighted-Greedy algorithm instead chooses an $(\x,m) \in \Omega \times \Mcal$ to optimize $\hat a_{k+1} = \Delta_{\x,m}/\lambda_m$ where $\Delta_{\x,m} = \bbI\left(\Ycal_k \cup \{\y_m(\x)\}, \y_M(\x'_l)|\Dcal\right) - \bbI\left(\Ycal_k, \y_M(\x'_l)|\Dcal\right)$ is the increase in mutual information by adding $(\x,m)$.  
    By scaling this $\Delta_{\x,m}$ value by $1/\lambda_m$ we can imagine splitting the effect of $(\x,m)$ into $\lambda_m$ copies of itself, and considering each of these copies as unit-weight elements.  We next argue that our Weighted-Greedy algorithm will achieve the same result as if we split each item into $\lambda_m$ copies, and that the process on these copies aligns with the standard setting.  
    
    First lets observe Weighted-Greedy will achieve the same result as if each $(\x,m)$ was split into $\lambda_m$ copies.  When we split each $(\x,m)$ into copies, each maintains the same scaled contribution of $\Delta_{\x,m}/\lambda_m$ to our objective function.  And we greedily add the item with largest contribution.  So if some $(\x,m)$ has the largest contribution $\hat a_{k+1}$ in the weighted setting, then so will its unit weight copy in the unweighted setting.  In the unit weight setting, after we add the first copy, this may effect the $\Delta_{\x',m'}/\lambda_{m'}$ contribution of some items $(\x',m') \in \Qcal_k$.  By submodularity, all such items have diminishing returns and their contribution cannot increase.  However, the unit weight copies of $(\x,m)$ are essentially independent, and so their $\Delta_{\x,m}/\lambda_m$ score does not decrease (if we add all $\lambda_m$ we increase mutual information by a total of $\Delta_{\x,m}$).  Since no other item can increase its score, and the copies scores do not decrease, if they were selected for having the maximal score, they will continue to have the maximal score until they are exhausted.  Hence, if we select one unit weight copy, we will add all of them consecutively, simulating the effect of adding the single weighted $(\x,m)$ at total cost $\lambda_m$.  Note that by our assumption in the theorem statement, we can always add all of them.  
    
    Finally, we need to argue that this unit-weight setting can invoke the submodular optimization approximation result.  For integer $\lambda_m$ and $B$ values, then this unit-weight version runs a submodular optimization with $B' < B$ steps. The acquisition function used to determine the greedy step is $\hat a_{k+1} = \Delta_{\x,m}/\lambda_m$, but since we have divided each item $(\x,m)$ into unit weight components with independent contribution to the mutual information $\frac{1}{A} \sum_{l=1}^A \bbI\left(\Ycal_k, \y_M(\x'_l)|\Dcal\right)$ they satisfy submodularity.  Then the weight is the same among all items so it can be ignored, and it maps to the standard submodular optimization with $B'$ steps, and achieves within $(1-1/e)$ of OPT($B'$) as desired.      
\end{proof}

\begin{corollary}[Corollary \ref{cor:sm-opt+}]
If Weighted-Greedy (Algorithm \ref{algo:algo}) is run until input-fidelity pair $(\x,m)$ that corresponds with the maximal acquisition function $\hat a_{k+1}(\x,m)$ would exceed the budget, it selects that input-fidelity pair anyways (the solution exceeds the budget $B$) and then terminates, the solution obtained is within $(1-1/e)$ of OPT($B$).  
\end{corollary}
\begin{proof}
    Consider that the extended Weighted-Greedy algorithm terminates using total $B^+ \geq B$ total budget.  By Theorem \ref{thm:near-sm-opt}, if we had $B^+$ budget, then this would achieve within $(1-1/e)$ of OPT($B^+$).  And since OPT($B^+$) $\geq$ OPT($B$), then this is within $(1-1/e)$ of OPT($B$) as well.  
\end{proof}

These results imply that the Weight-Greedy algorithm achieves the desired $(1-1/e)$-approximation until we are near the budget, or we slightly exceed it.  If the maximal weight item $\lambda_M$ is close to the full budget, then we are always in this unbounded case -- or may need to greatly exceed the budget to obtain a guarantee.  However, on the other hand, if $\lambda_M$ is fixed and the budget $B$ increases, then our bounds become more accurate.  In either case we can obtain a score within $(1-\frac{\lambda_M}{B})(1-1/e)$ of the OPT at a budget $B$ -- by excluding the part where the greedy choice may exceed the budget.  So as $\lambda_M/B$ goes to $0$, then the approximation goes to $(1-1/e)$.

While we have proven these results in the context of the specific approximated mutual information and parameter space $\Omega \times \Mcal$ these nearly $(1-1/e)$-optimal results will apply to any submodular optimization function, scaled by its optomal cost with a budget $B$.

Note that \citet{leskovec2007cost} proposed another approach to dealing with this budgeted submodular optimization.  They proposed to run two optimization schemes, one the method we analyze, and one that simply chooses the items that maximize $\Delta_{\x,m}$ at each step while ignoring the difference in their cost $\lambda_m$.  They show that while the first one may not achieve $(1-1/e)$-approximation, one of these schemes must achieve that optimality.  The cost of running both of them, however, is twice the budget, so in the worst case this combined scheme only achieves within $(1/2)(1-1/e)$ of the optimal.  This run-twice approach is also wasteful in practice, so we focused on showing what could be proven (near $(1-1/e)$-approximation) of just Weighted-Greedy.  In fact, as long as $\lambda_M/B \leq 1/2$, we already match their worst case bound.

\end{document}